\documentclass{article}

\usepackage{graphicx} 
\usepackage{natbib}
\usepackage{amsmath,amsfonts,amsthm}
\usepackage{multirow,rotating}
\usepackage{thmtools, thm-restate}
\usepackage{graphicx}
\usepackage{subcaption}
\usepackage{comment}
\usepackage{float}
\usepackage{mathtools}
\usepackage{amssymb}
\usepackage[preprint]{neurips_2025}

\newtheorem{lemma}{Lemma}[section]
\newtheorem{remark}{Remark}[section]

\usepackage[textsize=tiny]{todonotes}

\title{\texttt{C2-DPO}: Constrained Controlled \\ Direct Preference Optimization}

\author{%
  Kavosh Asadi\\
  Amazon\\
  \And
  Julien Han\\
  Amazon \\
  \AND
  Idan Pipano \\
  Technion \\
  \And
  Xingzi Xu \\
  Duke \\
  \And
  Dominique Perrault-Joncas \\
  Amazon \\
  \And
  Shoham Sabach \\
  Amazon \\
  \And
  Karim Bouyarmane\\
  Amazon \\
  \And
  Mohammad Ghavamzadeh \\
  Amazon \\
}

\begin{document}



\maketitle

\setlength{\parindent}{0pt}
\newcommand{\E}[1]{\mathbb{E}_{#1}}
\newcommand{\piRef}{\pi_{\textrm{ref}}}
\newcommand{\assign}{\mathrel{\mathop:}=}
\newcommand{\argmax}{\operatorname{arg\,max}} 

\begin{abstract}
Direct preference optimization (\texttt{DPO}) has emerged as a promising approach for solving the alignment problem in AI. In this paper, we make two counter-intuitive observations about \texttt{DPO}. First, we show that \texttt{DPO} loss could be derived by starting from an alternative optimization problem that only defines the KL guardrail on in-sample responses, unlike the original RLHF problem where guardrails are defined on the entire distribution. Second, we prove a surprising property of this alternative optimization problem, namely that under its optimal policy, both preferred and rejected responses tend to decrease in probability, a phenomenon typically displayed by DPO in practice. To control this behavior, we propose a set of constraints designed to limit the displacement of probability mass between the preferred and rejected responses in the reference and target policies. The resulting algorithm, which we call Constrained Controlled DPO (\texttt{C2-DPO}), has a meaningful RLHF interpretation. By hedging against the displacement, \texttt{C2-DPO} provides practical improvements over vanilla \texttt{DPO} when aligning several language models using standard preference datasets. 
\end{abstract}
\section{Introduction} \label{sec:intro}

Ensuring that AI systems act in accordance with human preferences, also known as the {\em alignment} problem, has become a critical focus in machine learning. Reinforcement Learning from Human Feedback (RLHF) has emerged as one promising approach~\citep{RLHF}. RLHF proceeds by first learning a reward model (RM), and then employing standard RL algorithms to maximize the RM while keeping the model close to a reference model. Recent years have witnessed the emergence of algorithms that solve the two RLHF sub-problems in a single step, chief among them being Direct Preference Optimization (\texttt{DPO}) algorithm~\citep{DPO}. \texttt{DPO} proceeds by leveraging the closed-form solution of the RLHF objective and using the preference dataset to align the model, thus obviating the explicit reward-learning, as well as the need to sample new responses during training. Since then, numerous extensions and successors have been proposed, e.g.,~\texttt{IPO}~\citep{IPO} and \texttt{CDPO}~\citep{CDPO}, underscoring the need for a deeper investigation into this emerging class of algorithms to connect the underlying principles.

In this paper, we start with a counter-intuitive observation about \texttt{DPO}: the \texttt{DPO} loss can be derived from an alternative optimization problem that imposes the KL penalty only on in-sample responses - those present in the preference dataset - rather than on the full output distribution, as done in traditional RLHF. We show that this subtle shift has a significant implication: the alternative optimization problem incentivizes in-sample probability reduction in DPO. We formally prove that under the optimal solution to this new problem, both preferred and rejected responses tend to decrease in probability. This phenomenon, while counter-intuitive, mirrors recent findings about DPO behavior (e.g.,~\citealt{nemotron,ppo_vs_dpo,smaug,distilled_DPO,xiao2024caldpo,AIPO,sppo,APO}), and is referred to as {\em likelihood displacement} by~\cite{razin2025unintentional}. We then show that the above in-sample probability reduction phenomenon is {\em shared among DPO extensions/successors} by developing a simple classification framework that unifies the family of DPO-style algorithms. 


Leaning on these insights, we propose a family of constraints that provably control likelihood displacement in DPO-style algorithms. The constraints are designed to limit the movement of winner-loser probability mass between the reference and target policies. Our proposed algorithm, Constrained Controlled DPO (\texttt{C2-DPO}), optimizes the \texttt{DPO} objective under these constraints, has a meaningful RLHF interpretation, and requires no extra computation. We evaluate the effectiveness of our constraints in enhancing preference alignment across two datasets and three models with up to 13B parameters, and show that \texttt{C2-DPO} outperforms vanilla \texttt{DPO} and several other baselines, delivering higher-quality final models when assessed holistically on the standard MT-Bench dataset~\citep{mt_bench}.

\section{Preliminaries} \label{sec:prelim}

We present the key ingredients of preference optimization on which we will build in the subsequent sections. In this setting, we are given a dataset $\mathcal{D}$ of triplets $(x , y_w , y_l)$, where $x$ is a prompt, while $y_w$ and $y_l$ reflect our preference in choosing response $y_w$ over $y_l$ conditioned on $x$. We are also given a reference policy $\piRef$ (often the SFT checkpoint $\pi_{\text{SFT}}$) which serves as a guardrail.

In RLHF, we first employ $\mathcal{D}$ to train a parameterized RM, $r_\phi$, and then use it to solve the following:
\begin{equation} \label{eq:standard_RLHF}
    \!\!\!\!\max_{\theta}\ \E{x}\Big[ \E{y \sim \pi_{\theta}}\big[r_\phi(x,y)\big] - \beta \mathbb{KL}\big(\pi_{\theta}(\cdot|x)||\piRef(\cdot|x)\big)\Big],
\end{equation}
where $\beta > 0$ is a hyper-parameter denoting the relative importance of reward maximization against ensuring a low deviation from $\piRef$. The RM is learned by minimizing the cross-entropy (CE) loss:
\begin{equation} \label{eq:reward-CE-loss}
\min_\phi \!\! \sum_{(x,y_w,y_l)\in\mathcal D}  \hspace{-0.15in} -\log\sigma\big(r_\phi(x,y_w) - r_\phi(x,y_l)\big)\;,
\end{equation}
assuming that preferences follow the Bradley-Terry (BT) model: $p(y_w \succ y_l \mid x) = \sigma\big(r(x,y_w) - r(x,y_l)\big)\;$ where $\sigma(x) = 1/(1+\exp(-x))$ is the sigmoid function and $r$ is the latent reward of the annotator. Fine-tuning $\pi_\theta$ in the RLHF approach is split into two stages: reward-learning using the BT model, followed by a policy optimization using~\eqref{eq:standard_RLHF}. More recently, a family of algorithms have emerged that solve the above two problems in a single stage: Direct Preference Optimization (DPO)-style algorithms. The loss function of \texttt{DPO} is derived from the RLHF problem \eqref{eq:standard_RLHF} using the recipe from~\citep{DPO}. The key insight here is that problem \eqref{eq:standard_RLHF} admits the following closed-form solution: $
\pi^*(y | x) = \piRef(y | x)\exp\big(r(x,y)/\beta\big)/Z(x)\ $ with $Z(x)$ as the partition function. We can rewrite this as
\vspace{-5pt}
\begin{equation} \label{eq:standard_reward}
r(x,y) = \beta\log \frac{Z(x)\pi^{*}(y | x)}{\piRef(y | x)}\ .
\end{equation}
%
Substituting $r(x,y)$ from~\eqref{eq:standard_reward} into~\eqref{eq:reward-CE-loss}, the partition function $Z(x)$ cancels out, leading to the optimization problem solved by \texttt{DPO}:
\vspace{-5pt}
\begin{equation} \label{eq:DPO}
\min_{\theta} \sum_{(x, y_w, y_l) \in \mathcal{D}} -\log \sigma\left(\beta\log \frac{\pi_{\theta}(y_w | x)}{\piRef(y_w | x)} - \beta\log \frac{\pi_{\theta}(y_l | x)}{\piRef(y_l | x)}\right)\;. 
\end{equation}

\section{KL in DPO: Implicit Guardrailing with a Counter-intuitive Side Effect} \label{sec:KLGuard}

Recall from the RLHF problem \eqref{eq:standard_RLHF} that large deviations from $\piRef$ are penalized by KL, and note that the penalty is applied to the entire distribution $\pi_{\theta}(\cdot|x)$, not only to samples from the dataset $\mathcal{D}$. One may wonder whether the KL guardrail is maintained in \texttt{DPO}, given the equivalence between \texttt{DPO} and RLHF shown in~\cite{DPO}. 

We now show that the standard \texttt{DPO} loss~(\ref{eq:DPO}) can be obtained by applying guardrails only to in-sample responses, highlighting the fact that DPO does not explicitly enforce KL beyond the data it is trained on. To this end, we replace the KL-penalty in~\eqref{eq:standard_RLHF} with a similar penalty, but one that only operates on in-sample responses $S_{x} = \{ y_w, y_l \mid (y_w,y_l,x) \in \mathcal{D}\}$:
\begin{equation} \label{eq:non-standard_RLHF}
    \!\!\!\!\max_{\theta}\ \E{x}\Big[ \E{y \sim \pi_{\theta}}\big[r_\phi(x,y)\big] - \beta\!\sum_{\textcolor{red}{y \in S_{x}}} \pi_{\theta}(y | x)\log \frac{\pi_{\theta}(y | x)}{\piRef(y | x)} \Big]\ .
\end{equation}
When comparing to the penalty term in~\eqref{eq:non-standard_RLHF}, recall that $\mathbb{KL}\big(\pi_{\theta}(\cdot|x)||\piRef(\cdot|x)\big) := \sum_{\textcolor{red}{y }} \pi_{\theta}(y | x)\log \frac{\pi_{\theta}(y | x)}{\piRef(y | x)}$.


We now prove that starting from problem~\eqref{eq:non-standard_RLHF} and following a recipe similar to the one in~\cite{DPO} described in Section~\ref{sec:prelim}, we ultimately arrive at the same standard \texttt{DPO} loss~\eqref{eq:DPO}. We report the proof of the following lemma in Appendix~\ref{Appendix_Idan}.

\begin{lemma} \label{L:DPO_non_standard}
Problem~\eqref{eq:non-standard_RLHF} has a closed-form solution. While this closed-form solution is different than the closed-form solution to problem~\eqref{eq:standard_RLHF}, substituting it into the BT model and following the recipe of~\cite{DPO} leads to the same standard \texttt{DPO} loss in~\eqref{eq:DPO}.
\end{lemma}

At first glance, the two optimization problems~\eqref{eq:standard_RLHF} and \eqref{eq:non-standard_RLHF} look quite similar, but finding the closed-form solution of~\eqref{eq:non-standard_RLHF} requires a delicate analysis. In particular, the new penalty is only summing over the set $S_x$, so the resultant KKT optimality conditions are more involved (details in Appendix~\ref{Appendix_Idan}).

Note that the introduced penalty in \eqref{eq:non-standard_RLHF} is not even a proper divergence, in the sense that it could be negative. Thus, the lemma could be viewed as evidence that when we move to \texttt{DPO}, we no longer enforce the $\mathbb{KL}$ penalty explicitly. However, empirically we still see that DPO can effectively guardrail using KL since in practice, using larger values of $\beta$ results in smaller KL deviations during \texttt{DPO} training (see Figure~\ref{fig:KL}-Left). This arguably makes sense, because in DPO there is little incentive for the model to shift probability mass for responses that are quite different from in-sample responses, and so explicit out-of-sample guardrailing may not be necessary.

\begin{figure*}[h]  
   \centering
   \begin{subfigure}{\textwidth}
       \centering
       \includegraphics[width=.4\textwidth]{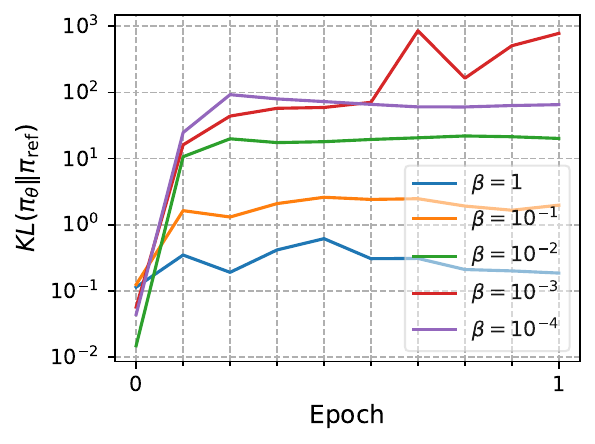}
       \hspace{0.125\textwidth}
       \includegraphics[width=.4275\textwidth]{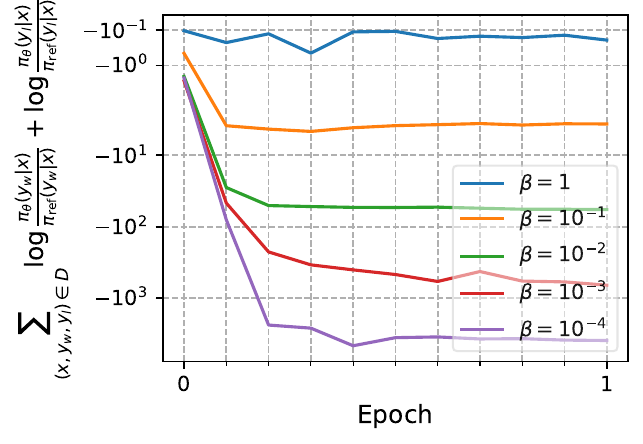}
   \end{subfigure}
   \caption{{\bf Left:} Increasing $\beta$ in DPO leads to effective out-of-sample guardrailing. We ran DPO with different values of $\beta$ starting from the Zephyr-7B initial checkpoint on the UltraFeedback dataset. We then estimated the KL divergence between $\pi_{\theta}$ and $\piRef$ by autoregressively sampling $N=32$ responses from $\pi_{\theta}$ for each prompt in the test set, followed by computing $\frac{1}{N}\sum_{i=1}^{N} \log \frac{\pi_{\theta}(y_i|x)}{\piRef(y_i|x)}$ and averaging over prompts. {\bf Right:} Reduction of in-sample probabilities in DPO training.}
   \label{fig:KL}
\end{figure*}

While this minimal guardrailing is effective, we show that it nevertheless leads to counter-intuitive behavior. Notice again that the new penalty term in \eqref{eq:non-standard_RLHF} can be negative, in sharp contrast to the original KL term, which is non-negative by definition. More importantly, the overall objective in~\eqref{eq:non-standard_RLHF} can be increased by decreasing the probability of both winner and loser responses, making the new penalty term negative, i.e.,~$\log \pi_\theta(y \mid x)/\pi_{\text{ref}}(y \mid x) < 0, \; \forall y \in S_x$. Note that this behavior is not incentivized in the original optimization problem~\eqref{eq:standard_RLHF}. We know that the KL penalty in~\eqref{eq:standard_RLHF} is always non-negative, so even if we reduce the in-sample portion of the KL by decreasing in-sample probabilities, the out-of sample portion must get more and more positive, and so there is no point in blindly reducing the in-sample probabilities. We now formalize our intuitions: 

\begin{lemma}\label{lemma: wierd e^-1 phenomenon}
Let $(x,y)$ be an in-sample prompt-response pair, i.e.,~$x\in\mathcal D$ and $y\in S_x$. Suppose that $r_{\phi}(x , y) \leq \max_{y' \notin S_{x}} r_{\phi}(x,y')$. Then, any optimal solution $\theta$ to the optimization problem \eqref{eq:non-standard_RLHF} satisfies $\pi_{\theta}(y | x) \leq e^{-1} \piRef(y | x)$. 
\end{lemma}

The proof, reported in Appendix~\ref{Appendix_Idan}, hinges on the fact that the penalty term in \eqref{eq:non-standard_RLHF} is defined on in-sample responses. Therefore, this result does not hold when solving problem~\eqref{eq:standard_RLHF}, indicating that the standard RLHF formulation is not susceptible to this counter-intuitive behavior. To better understand the result, note that for a prompt $x$, if the reward of an in-sample response (whether the response is preferred or rejected) is smaller than the maximal reward of out-of-sample responses, then any optimal solution of \eqref{eq:non-standard_RLHF} decreases the probability of this in-sample response. Clearly the size of the in-sample response set $S_x$ is relatively small in comparison to the rest of the set (in the extreme case, a single preferred and rejected response), so the condition is likely to hold.

Interestingly, it has been observed recently that during DPO training all in-sample probabilities - even those associated with preferred responses - tend to decrease in magnitude (e.g.,~\citealt{nemotron,ppo_vs_dpo,smaug,distilled_DPO,xiao2024caldpo,NEURIPS2024_fa69e968,dposparse,dpotodo,yuzi2025identifying,razin2025unintentional,dpoless,dposhift,huang2025correcting,dpobalancing}). We also observe this clear trend in our experiments, as is apparent in Figure~\ref{fig:KL} (Right) where we show that the sum of log ratios ($\log \frac{\pi_{\theta}(y_w|x)}{\piRef(y_w|x)} + \log \frac{\pi_{\theta}(y_l|x)}{\piRef(y_l|x)}$) tends to decrease radically. Lemma~\ref{lemma: wierd e^-1 phenomenon} 
hints at the underlying reason for this counter-intuitive behavior, which has been referred to as {\em likelihood displacement} of in-sample probabilities~\citep{razin2025unintentional}. To the best of our knowledge, while this has been reported in previous empirical studies, Lemma~\ref{lemma: wierd e^-1 phenomenon} is among very few theoretical results explaining this counter-intuitive phenomenon.

We conclude this section by noting that we proved in-sample probability reduction for DPO alone, and so it would be natural to ask if some of its main successors, e.g.,~\texttt{IPO}~\citep{IPO} and \texttt{CDPO}~\citep{CDPO}, share the same property. In the next section, we answer this question affirmatively by developing a classification framework that unifies the family of DPO-style algorithms.
\section{A Classification View of DPO-style Algorithms}
\label{sec:PO-C}

We now show that DPO-style algorithms can be interpreted as classification methods, where the objective is defined solely over in-sample responses. As a result, similar to DPO, none of these algorithms applies any direct guardrailing on out-of-sample responses. Recall that the standard classification setting has three main ingredients.

\textbf{First}, we construct a hypothesis space by defining probabilities assigned to each class. In DPO-style algorithms, these probabilities are implicitly defined as:
\begin{equation}
p_{\theta}(x,y_w,y_l) \assign \textrm{softmax}\big(  r_{\theta}(x , y_w) , r_{\theta}(x , y_l)\big)\ ,
\label{eq:softmax_binary}
\end{equation}
where $r_{\theta}$ is the reward defined in \eqref{eq:standard_reward} having substituted $\pi^{*}$ with $\pi_{\theta}$. Under $p_{\theta}$, the probability assigned to the {\em winner (preferred)} response $y_w$, denoted by $p^{ w}_{\theta}$, does not depend on the partition function $Z(x)$ and can be written as:\vspace{-10pt}
\begin{equation} \label{pthetaw}
    p^{w}_{\theta}(x,y_w,y_l) := \frac{\big(\frac{\pi_{\theta}(y_w \mid x)}{\piRef(y_w \mid x)}\big)^{\beta}}{\big(\frac{\pi_{\theta}(y_{w} \mid x)}{\piRef(y_{w} \mid x)}\big)^{\beta} + \big(\frac{\pi_{\theta}(y_{l} \mid x)}{\piRef(y_{l} \mid x)}\big)^{\beta}} \ .
\end{equation}
%
The probability assigned to the {\em loser (rejected)} response $y_l$, denoted by $p^{l}_{\theta}$, can be defined similarly. 
Note that the distribution $p_\theta$ in~\eqref{eq:softmax_binary} can be thought of as a generalization of the conditional probability of a response $y$ given that $y \in \{y_w, y_l\}$.

\textbf{Second}, in the standard classification setting, 
the dataset gives us access to labels, which we use to extract target probabilities associated with each class. To obtain these target probabilities, we simply use any distribution $p=(p^w , p^l)$ from the simplex $\Delta_{2}$, which is defined as the set of all vectors $p\in\mathbb{R}^{2}$ satisfying $p^w,p^l \geq 0$ and $p^w + p^l =1$. In the most basic case, we just use the one-hot vector $(p^w , p^l) = (+1 , 0)$ akin to using {\em hard labels}. More generally, we can use {\em soft labels}, meaning that we put some non-zero weight behind each class~\citep{softLabels}.

\textbf{Third}, we define a classification loss $\mathcal L$ between two distributions $p_{\theta}$ and $p$, leading us to the optimization problem: $\min_{\theta} \sum_{\mathcal{D}}{\mathcal L}(p_{\theta},p)$. A good example is the CE loss.

We can now show that a large number of DPO-style algorithms can be viewed as specific instances of this classification framework. The generality arises from the ability to use {\bf (a)} hard or soft labels for the target distribution $p$ and {\bf (b)} different classification losses $\mathcal{L}$.
\begin{remark}[\texttt{DPO}]
Suppose that we use the CE loss and hard labels $p \assign (p^{w} , p^{l}) = (+1 , 0)$ in the above classification framework. Then, using \eqref{pthetaw}, we can write
\begin{equation*}
{\mathcal L}\big(p_{\theta},\ p\big)\! =\! -\left(p^w \log p^{w}_{\theta} + p^l \log p^{l}_{\theta}\right)\! =\! -\log p^{w}_{\theta} = {\displaystyle -\log \sigma\left(\beta\log \frac{\pi_{\theta}(y_w | x)}{\piRef(y_w | x)} - \beta\log \frac{\pi_{\theta}(y_l |x)}{\piRef(y_l | x)}\right)}\ ,
\end{equation*}
which is exactly the \texttt{DPO} loss~\eqref{eq:DPO} if it is summed over $\mathcal{D}$.
\end{remark}

Another popular DPO-style algorithm is \texttt{IPO}~\citep{IPO}. While the derivation of \texttt{IPO} in the original paper looks completely different than \texttt{DPO}, we now show that \texttt{IPO} can also be viewed as a specific instance of our classification framework (see Appendix~\ref{Appendix:IPO} for the detailed derivation).

\begin{remark}[\texttt{IPO}]
We can recover \texttt{IPO} (Eq.~17 in~\citealt{IPO}) using the loss $\mathcal{L}\big(p_{\theta},\ p\big) = \big(\log (p^{w}_{\theta}/p^{l}_{\theta}) - \log (p^{w}/p^{l})\big)^2$ and soft labels $p \assign (p^{w}, p^{l}) = (\sigma(1/2) , \sigma(-1/2))$. 
\end{remark}

Table~\ref{table_example_algorithms} shows that several DPO-style algorithms can be formulated using this framework. Given this framework, we can formulate the set of optimal solutions (those that achieve 0 loss) for any DPO-style algorithm. An optimal parameter $\theta$ is one that achieves $0$ loss for all samples in $\mathcal D$, i.e.,~$p_{\theta}(x,y_w,y_l) = p, \forall(x,y_w,y_l)\in\mathcal D$. Setting $p_{\theta}^w(x,y_w,y_l)$ in~\eqref{pthetaw} equal to $p^w=1-\varepsilon$, we obtain \vspace{-5pt}
\begin{equation} \label{DPO_solution_characterization}
    \pi_{\theta^{*}}(y_w | x) = \eta\cdot \pi_{\theta^{*}}(y_l | x), \quad\; \textrm{with}\;\;\; \eta \assign \sqrt[\beta]{(1 - \varepsilon)/\varepsilon}\cdot\frac{\piRef(y_w | x)}{\piRef(y_l | x)} \ .
\end{equation}
Note that the derivation using $y_l$ yields the same result. Thus, we have two probabilities, $\pi_{\theta^{*}}(y_w | x)$ and $\pi_{\theta^{*}}(y_l | x)$, that we aim to learn, but minimizing the loss only gives us one constraint specified in~\eqref{DPO_solution_characterization}. This means that the original classification problem (loss-minimization in DPO-style algorithms) is under-specified. 
In Figure~\ref{solution_set}, we provide an illustration of this phenomenon. We also give a concrete example in Appendix~\ref{appendix_concrete_example} in which the two probabilities can move in arbitrary directions, and most notably, they can both go to zero.

\begin{table}[t]
    \centering
    \begin{minipage}{0.7\textwidth}
    \begin{tabular}{|l||c||c|c|}
        \hline
        Data & Algorithm & Labels & Loss \\ \hline\hline
        \multirow{3}{*}{Pairs} & \texttt{DPO(BT)} & Hard & CE\\ \cline{2-4}
         & \texttt{CDPO} & Soft & CE \\ \cline{2-4}
         & \texttt{IPO} & Soft & Remark 4.2 \\ \hline\hline
        List (Appendix.~\ref{Appendix:Algo}) & \texttt{DPO(PL)} & Hard & CE \\ \hline\hline
         Auxiliary Info & \texttt{RPO} & Soft & Appendix~\ref{Appendix:Algo} \\ \cline{2-4}
         (Appendix.~\ref{Appendix:Algo}) & \texttt{Distilled DPO} & Soft & Appendix~\ref{Appendix:Algo} \\ \hline
    \end{tabular}
    \end{minipage}%
    \begin{minipage}{0.29\textwidth}
    \caption{Unifying DPO-style algorithms. By \texttt{DPO(PL)}, we mean \texttt{DPO} with Plackett-Luce model. See our proofs for \texttt{RPO}~\citep{nemotron} and \texttt{Distilled DPO}~\citep{distilled_DPO} in Appendix \ref{Appendix:Algo}.}    \label{table_example_algorithms}
    \end{minipage}
\end{table}

\begin{figure}[t]
    \centering
    \begin{minipage}{0.35\textwidth}
    \includegraphics[width=\textwidth]{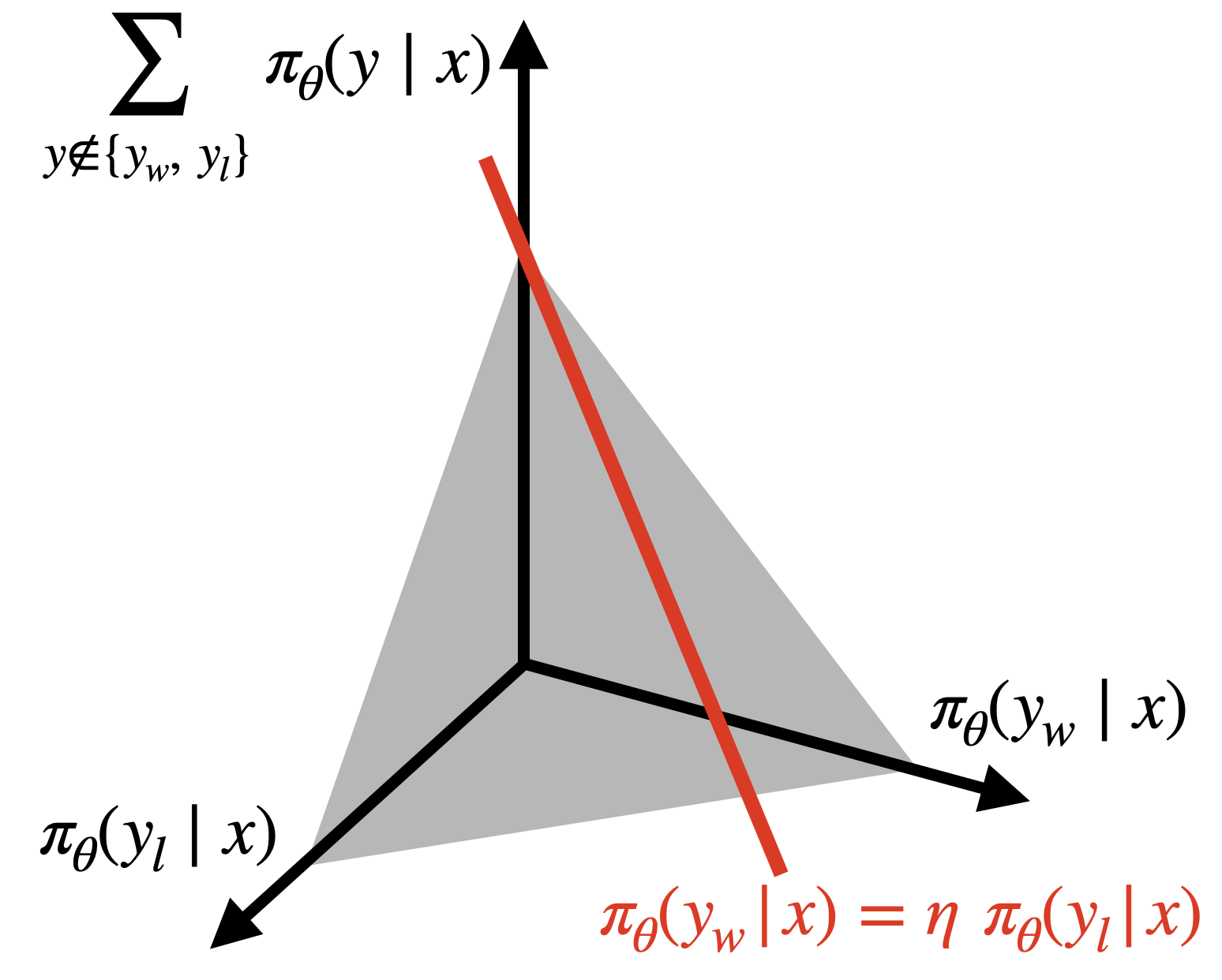}
    \end{minipage}%
    \hspace{0.06\textwidth}
    \begin{minipage}{0.5\textwidth}
    \caption{An illustration of the set of solutions that achieve $0$ loss in DPO-style algorithms. The shaded gray is the set of feasible solutions. The red line passing through the feasible set indicates the set of optimal solutions. Note that the case where the probability belonging to in-sample responses displace entirely to out-of-sample responses, i.e.,~$p_{\theta}(x, y_w, y_l)=( 0^{+}, 0^{+} )$ also lies on this line.}
    \label{solution_set}
    \end{minipage}
\end{figure}

\section{Constrained Controlled DPO (\texttt{C2-DPO})}
\label{sec:algos} 

We now present a general family of constraints to control the likelihood displacement of in-sample probabilities in DPO-style algorithms described in Section~\ref{sec:KLGuard}. These constraints can also help with the under-specified nature of these algorithms discussed in Section~\ref{sec:PO-C}. We define the constraints on the probability mass of the winner-loser pair and use them to control how much this mass changes from the reference policy $\piRef$ to the target policy $\pi_\theta$. We then show how these constraints can be incorporated into any DPO-style loss function and propose our algorithm, which we refer to as Constrained Controlled Direct Preference Optimization (\texttt{C2-DPO}).  

The constraint, with respect to an arbitrary function $\varphi : \mathbb{R} \rightarrow \mathbb{R}$, takes the general form
\begin{equation} \label{P:Cons}
\varphi\big(\pi_{\theta}(y_w|x)\big) + \varphi\big(\pi_{\theta}(y_l|x)\big) = \varphi\big(\piRef(y_w|x)\big) + \varphi\big(\piRef(y_l|x)\big)\ . 
\end{equation} 
%
Note that the RHS is fixed during training, so the two terms on the LHS cannot move in the same direction. We now generalize this intuition by showing that when the constraint function $\varphi$ is monotonic and added to the solution characterization of DPO-style algorithms given by
\vspace{-5pt}
\begin{equation} \label{DPO_solution_characterization1}
\frac{\pi_{\theta^{*}}(y_w|x)}{\piRef(y_w|x)} = \sqrt[\beta]{\frac{1-\varepsilon}{\varepsilon}} \cdot \frac{\pi_{\theta^{*}}(y_l|x)}{\piRef(y_l|x)}\ ,    
\end{equation}
then we can control the direction of the movement of probability mass for all winner-loser pairs.


\begin{restatable}[]{proposition}{tech}
\label{prop:monotonic}
    Let $\varphi : \mathbb{R} \rightarrow \mathbb{R}$ be a monotonic function and assume that~\eqref{P:Cons} holds.
    Then, $\pi_{\theta^{*}}(y_w|x) > \piRef(y_w|x)$ and $\pi_{\theta^{*}}(y_l|x) < \piRef(y_l|x)$.
\end{restatable}    

Proposition~\ref{prop:monotonic}, whose proof is reported in Appendix~\ref{Appendix:C2-DPO-Proposition}, shows that any monotonic constraint function $\varphi$ guarantees that the learned policy, $\pi_{\theta^*}$, assigns a higher (lower) probability to the winner (loser) response than the one assigned to it by the reference policy $\piRef$. It is natural to ask what $\varphi$ should be used in the context of this constraint, and we present two interesting candidates below. 
\subsection{Logarithmic Constraint $\;\varphi(x) \assign \log x$}
\label{subsec:Log-constraint}
Our first choice is to employ the logarithmic constraint:
\begin{equation}
\log\left(\pi_{\theta}(y_w|x)\right) + \log\left(\pi_{\theta}(y_l|x)\right) =  \log\left(\piRef(y_w|x)\right) + \log\left(\piRef(y_l|x)\right)\ ,
\label{eq:constraint_phi_log}
\end{equation}
which is nice to work with empirically in light of the fact that all terms are in the log-space. Moreover, these log probabilities are already computed in \texttt{DPO}, which makes the implementation of the corresponding \texttt{C2-DPO} algorithm more efficient. 

Rather than using hard constraints, it is easier to compute the deviation from the constraint using either $\ell_1$ or $\ell_2$ norm, and then add it as a regularizer to the original DPO-style loss with a regularization parameter $\lambda$ that trades-off the relative importance of the two terms. 
Equipping the \texttt{DPO} loss~\eqref{eq:DPO} with the logarithmic constraint~\eqref{eq:constraint_phi_log}, we obtain the following loss for \texttt{C2-DPO}:
\begin{equation}
\!\!\!\!\min_{\theta}\sum_{\mathcal{D}} \! -\log \sigma\left(\beta\log\frac{\pi_{\theta}(y_w|x)}{\piRef(y_w|x)} \!-\! \beta\log\frac{\pi_{\theta}(y_l|x)}{\piRef(y_l|x)}\right) + \lambda\left(\log\frac{\pi_{\theta}(y_w|x)}{\piRef(y_w|x)} \!+\! \log\frac{\pi_{\theta}(y_l|x)}{\piRef(y_l|x)}\right)^{2}\!\!\!.
\label{eq:C2DPO-log}
\end{equation}
In contrast to the hard constraint, in this case we do not necessarily force the winner (loser) probability to go up (down). Rather, we impose a penalty when the learner violates the constraint. Notice also that we added the penalty term to the original loss of \texttt{DPO} for simplicity, but in principle, the penalty term can be added to any DPO-style loss covered in our classification framework.

Further, we can show that employing the logarithmic constraint~\eqref{eq:constraint_phi_log} has a meaningful RLHF interpretation. Recall that~\citet{DPO} defined $\hat r_{\theta}(x,y) := \beta\log\big(\pi_{\theta}(y|x)/\piRef(y|x)\big),\;\forall y\in\mathcal Y$ 
as an {\em implicit reward} learned during DPO training. Using this notation, we can rewrite objective~\eqref{eq:C2DPO-log} simply as
\vspace{-10pt}
\begin{equation*}
-\log \sigma\big( \hat r_{\theta}(x , y_w) - \hat r_{\theta}(x , y_l)\big) + \frac{\lambda}{\beta^{2}}\big( \hat r_{\theta}(x , y_w) +  \hat r_{\theta}(x , y_l)\big)^{2}\ .
\end{equation*}
Under $\varphi(x) = \log x$, we solve the original RLHF problem akin to \texttt{DPO}, but we also incentivize the sum of the implicit rewards for the winner and loser to remain around zero. It follows that the constraint regularizes the implicit rewards so as to avoid rewards that are {\bf (a)} very large and {\bf (b)} have the same sign. These two properties cannot co-exist when employing $\varphi(x) = \log x$, since doing so would yield a large magnitude inside the square and ultimately a large magnitude in the second term of the loss. Intuitively, this can hedge against the likelihood displacement, because in the case of displacement both implicit rewards are large in magnitude and both have a negative sign, which the constraint will greatly penalize.

\subsection{Identity Constraint $\;\varphi(x) \assign x$}
\label{subsec:Identity-constraint}

A second interesting choice would be to simply use the identity constraint:
\begin{equation}
    \pi_{\theta}(y_w|x) + \pi_{\theta}(y_l|x)  = \piRef(y_w|x) + \piRef(y_l|x)\ .
    \label{eq:constraint_phi_identity}
\end{equation}
While~\eqref{eq:constraint_phi_identity} is also a plausible constraint, at first glance it is unclear how to implement it since the constraint is no longer in the log-space and is specified in terms of raw probabilities. Working with raw probabilities is prone to numerical underflow issues; thus, we would like to derive a constraint which is equivalent to~\eqref{eq:constraint_phi_identity} and operates in the log-space. To do so, we make use of the following lemma whose proof is reported in Appendix~\ref{Appendix:C2-DPO-Lemma}.

\begin{restatable}[]{lemma}{technical} \label{L:Technical}
For any two numbers $a$ and $b$, we have $\log(a + b) = \log a - \log\sigma(\log a - \log b)$. 
\end{restatable}

Applying $\log$ to both sides of~\eqref{eq:constraint_phi_identity}, we obtain  $\log(\pi_{\theta}(y_w|x) + \pi_{\theta}(y_l|x)) = \log(\piRef(y_w|x) + \piRef(y_l|x))$,
which can be rewritten using Lemma~\ref{L:Technical} as
\begin{align}
\label{eq:v2}
\log\big(\pi_{\theta}(y_w|x)\big) - \log\sigma\left(\log\frac{\pi_{\theta}(y_w|x)}{\pi_\theta(y_l|x)}\right) =\log\big(\piRef(y_w|x)\big) - \log\sigma\left(\log\frac{\piRef(y_w|x)}{\piRef(y_l|x)}\right)\ .
\end{align}
Moving from~\eqref{eq:constraint_phi_identity} to~\eqref{eq:v2}, we have rewritten the constraint entirely in the log-space, thus avoiding numerical issues, and similar to the logarithmic constraint in Section~\ref{subsec:Log-constraint}, allowing a straightforward implementation of the corresponding \texttt{C2-DPO} algorithm. Equipping the \texttt{DPO} loss~\eqref{eq:DPO} with the logarithmic constraint~\eqref{eq:v2}, we obtain the following loss for \texttt{C2-DPO}: 
\begin{align}
&\min_{\theta}\;\sum_{\mathcal{D}} \; -\log \sigma\left(\beta\log\frac{\pi_{\theta}(y_w|x)}{\piRef(y_w|x)} - \beta\log\frac{\pi_{\theta}(y_l|x)}{\piRef(y_l|x)}\right) \nonumber \\ 
&\qquad\qquad +\lambda\left(\log\frac{\pi_{\theta}(y_w|x)}{\piRef(y_w|x)} + \log\sigma\left(\log\frac{\piRef(y_w|x)}{\piRef(y_l|x)}\right) - \log\sigma\left(\log\frac{\pi_{\theta}(y_w|x)}{\pi_\theta(y_l|x)}\right)\right)^{2}\ .
\label{eq:C-3DPO-log}
\end{align}
Note that unlike $\varphi(x)=\log x$, deriving an RLHF interpretation under $\varphi(x)= x$ is subtle. That said, an interesting property under $\varphi(x)= x$ is that the winner probability can increase only by an amount equal to $\piRef(y_l|x)$. This means that we will not put the entire probability mass on $y_w$.
\section{Experiments}
In this section, we present experiments on two data sets that demonstrate that \texttt{C2-DPO} outperforms vanilla \texttt{DPO} and several other baselines and delivers higher-quality final models when assessed holistically on the standard MT-Bench dataset~\citep{mt_bench}.

\begin{figure*}[h]
    \centering
    \begin{subfigure}[t]{0.24\textwidth}
        \centering
        \includegraphics[width=\linewidth, trim={0pt 15pt 0pt 0pt}, clip]{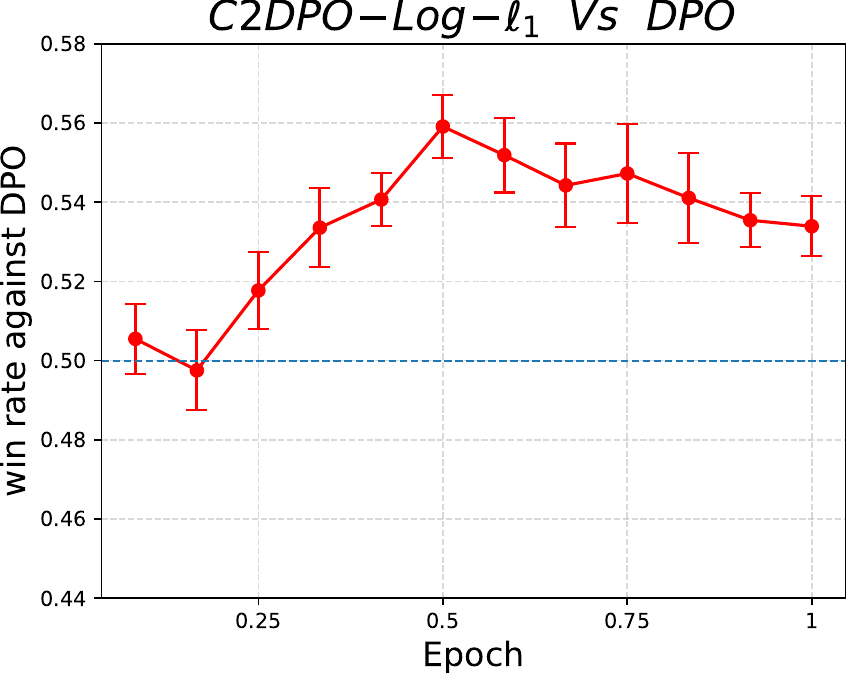}
    \end{subfigure}
    \hfill
    \begin{subfigure}[t]{0.23\textwidth}
        \centering
        \includegraphics[width=\linewidth, trim={0pt 15pt 0pt 0pt}, clip]{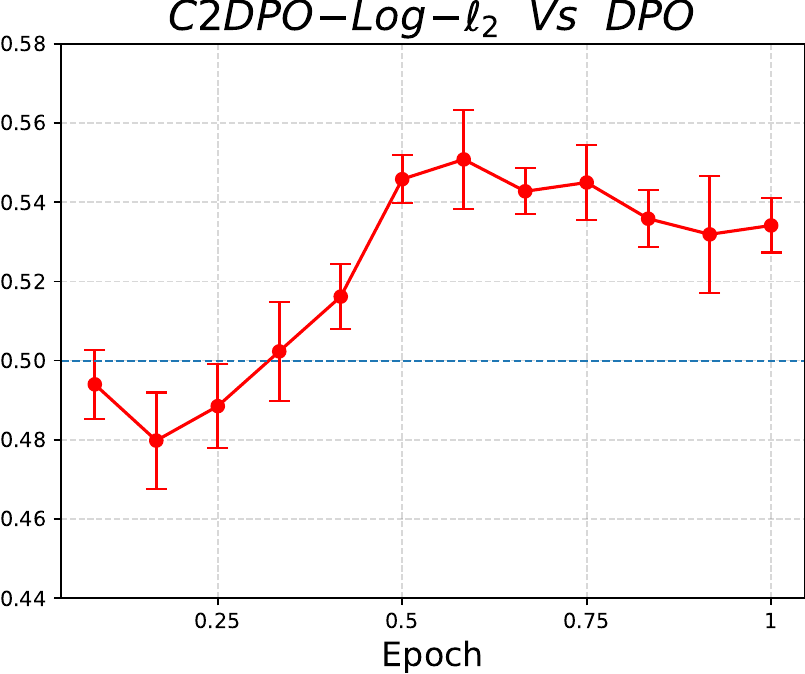}
    \end{subfigure}
    \hfill
    \begin{subfigure}[t]{0.23\textwidth}
        \centering
        \includegraphics[width=\linewidth, trim={0pt 15pt 0pt 0pt}, clip]{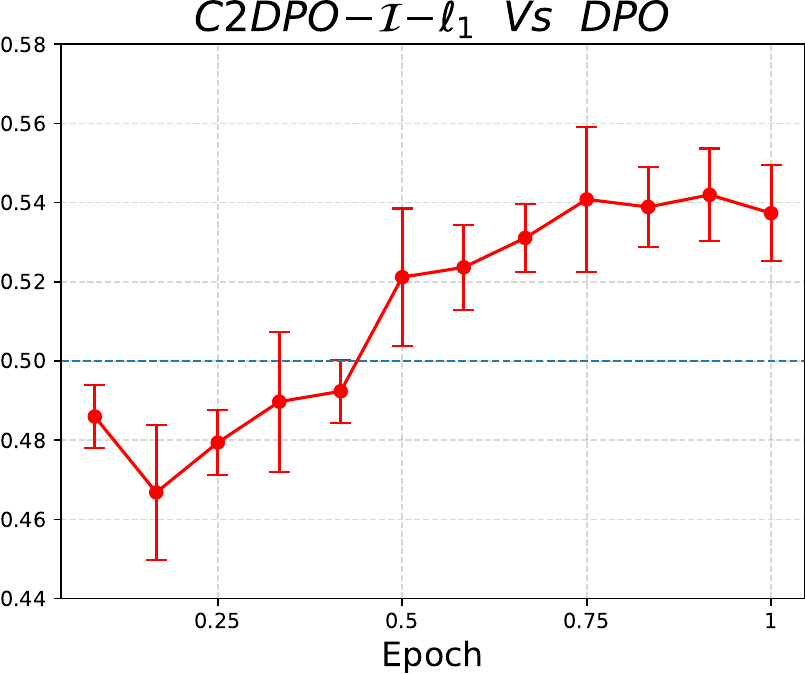}
    \end{subfigure}
    \hfill
    \begin{subfigure}[t]{0.23\textwidth}
        \centering
        \includegraphics[width=\linewidth, trim={0pt 15pt 0pt 0pt}, clip]{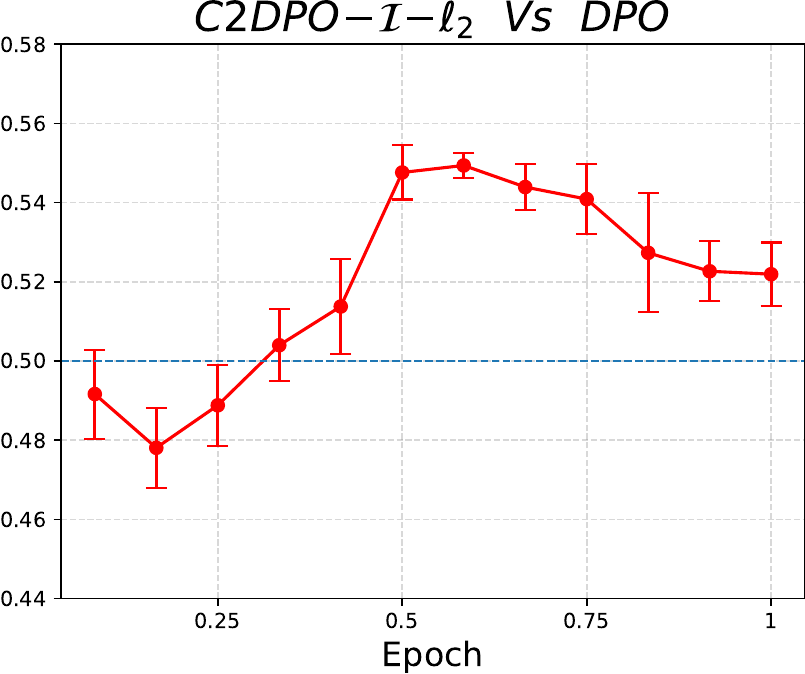}
    \end{subfigure}
    \caption{Head-to-head win-rate of \texttt{C2-DPO} against \texttt{DPO}, averaged over 10 random seeds.}
    \label{fig:four_side_by_side}
\end{figure*}


\begin{figure*}[h]
    \centering
    \begin{minipage}{0.65\textwidth}
        \centering
        \includegraphics[width=0.48\textwidth, trim={0pt 5pt 0pt 0pt}, clip]{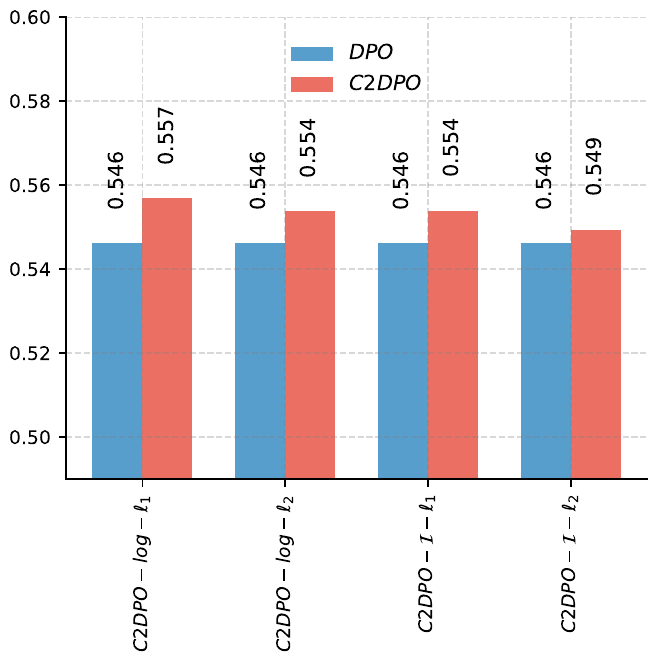}
        \hspace{0.02\textwidth}
        \includegraphics[width=0.45\textwidth, trim={0pt 5pt 0pt 0pt}, clip]{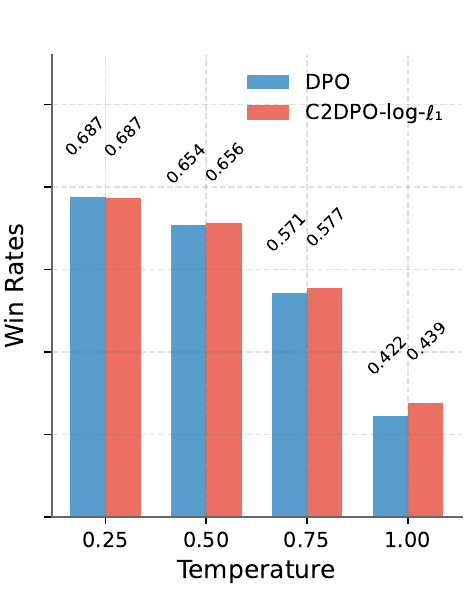}
    \end{minipage}%
    \begin{minipage}{0.3\textwidth}
        \captionof{figure}{\textbf{Left}: win-rate against the preferred response in the test set of the UltraFeedback dataset. \textbf{Right}: Win rates on the TL;DR dataset. See Appendix for more details.}
        \label{fig:uf_and_tldr}
    \end{minipage}
\end{figure*}

\begin{figure*}[h]  
   \centering
   \begin{subfigure}{\textwidth}
       \centering
       \includegraphics[width=.85\textwidth]{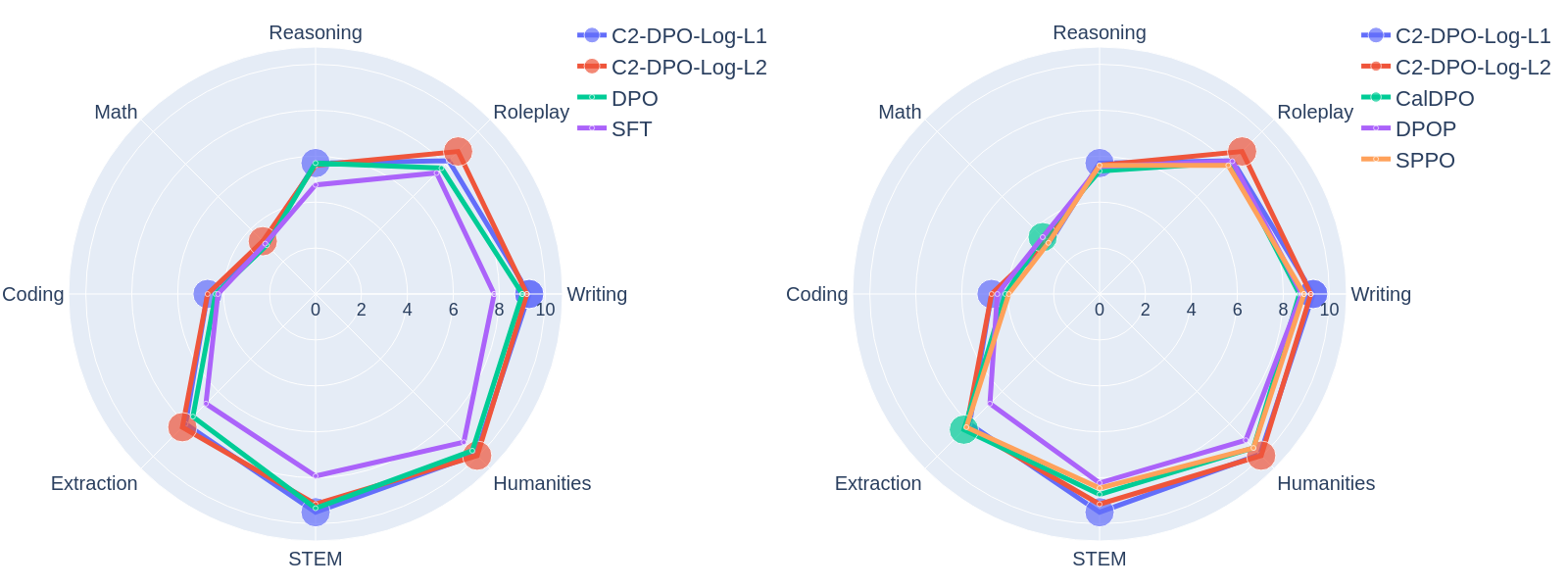} 
   \end{subfigure}
   \caption{A comparison between \texttt{C2-DPO} and baselines. (Left) The two variants of \texttt{C2-DPO-Log} better align \textrm{Zephyr-7b-SFT} relative to vanilla \texttt{DPO} as well as the other baselines (Right).}
   \label{fig:uf_7b_radar}
\end{figure*}
\subsection{\textrm{Ultrafeedback Binarized}}
This dataset is comprised of 64k prompts, and a winner (preferred) and loser (rejected) continuation per prompt ($y_w$ and $y_l$). We used the standard DPO implementation published with the paper~\citep{mitchell2023dpo_implementation}. Following~\citet{rasul2024preference}, we train \textrm{Zephyr-7B-SFT}. 
\begin{figure}[h]  
   \centering
   \begin{minipage}{0.6\textwidth}
   \begin{subfigure}{0.45\textwidth}
       \centering
       \includegraphics[width=\textwidth]{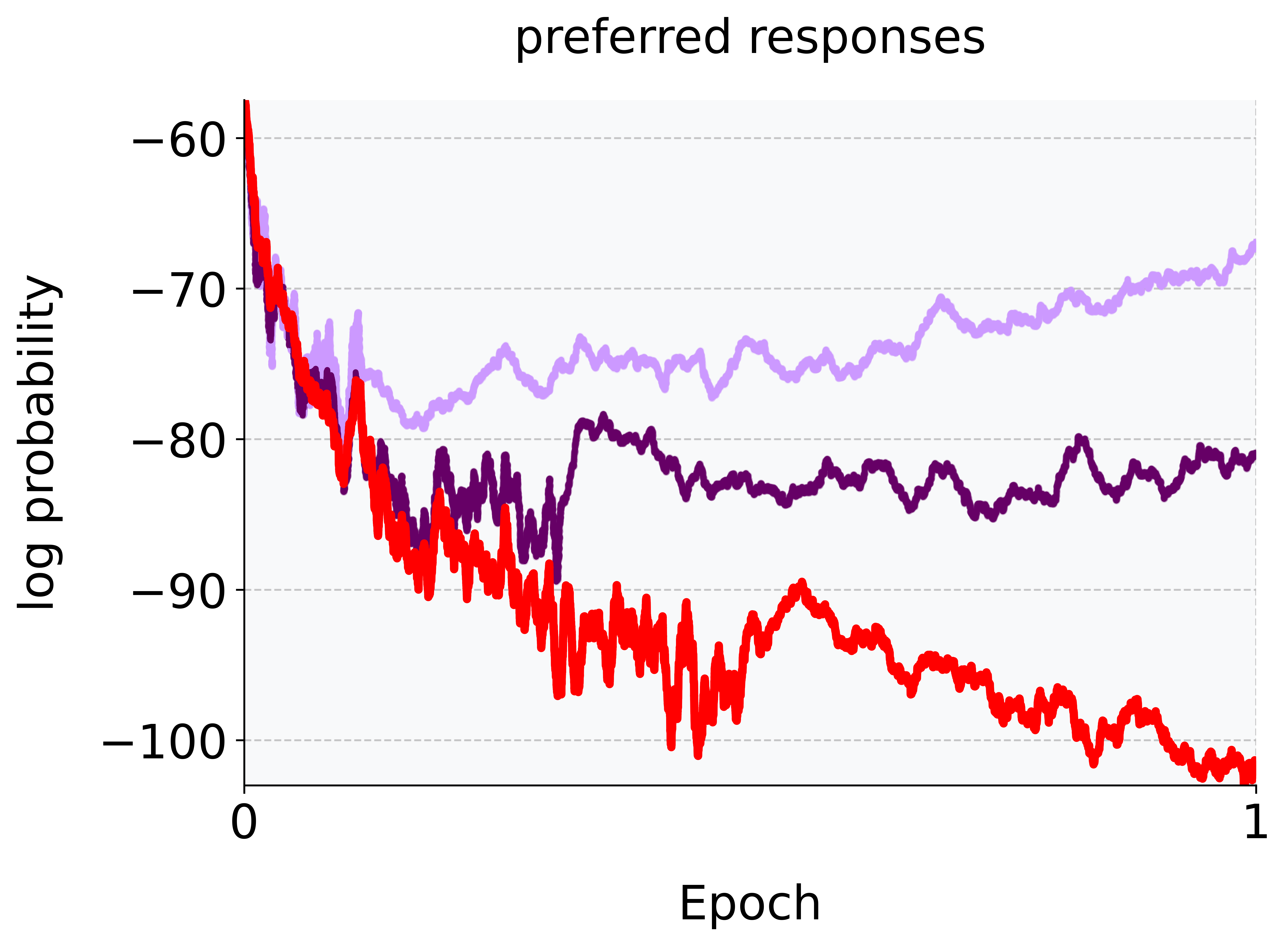}
   \end{subfigure}
   \begin{subfigure}{0.45\textwidth}
       \centering
       \includegraphics[width=\textwidth]{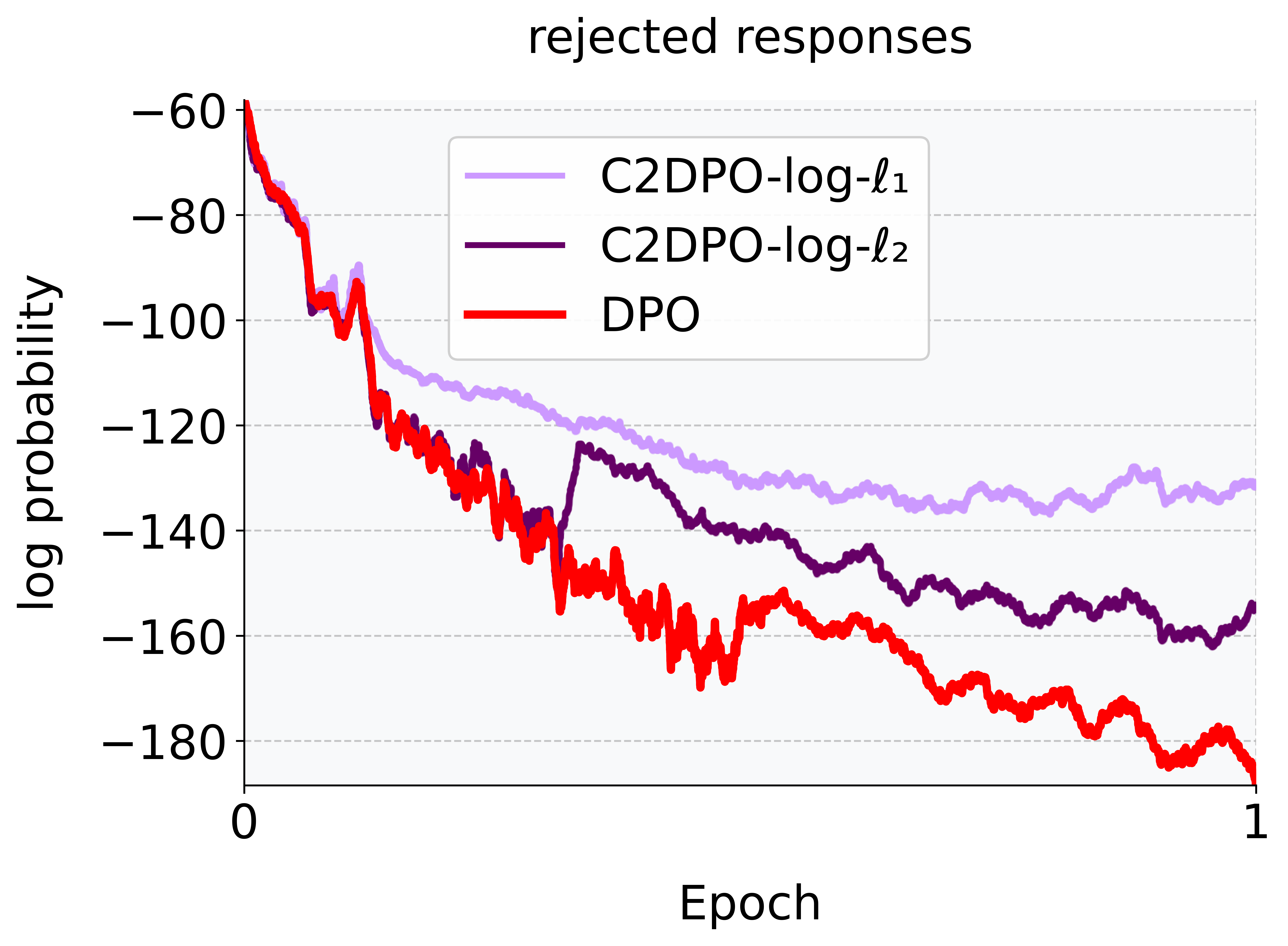}
   \end{subfigure}
   \end{minipage}
   \begin{minipage}{0.39\textwidth}
   \begin{subfigure}{0.95\textwidth}
       \centering
       \includegraphics[width=\textwidth]{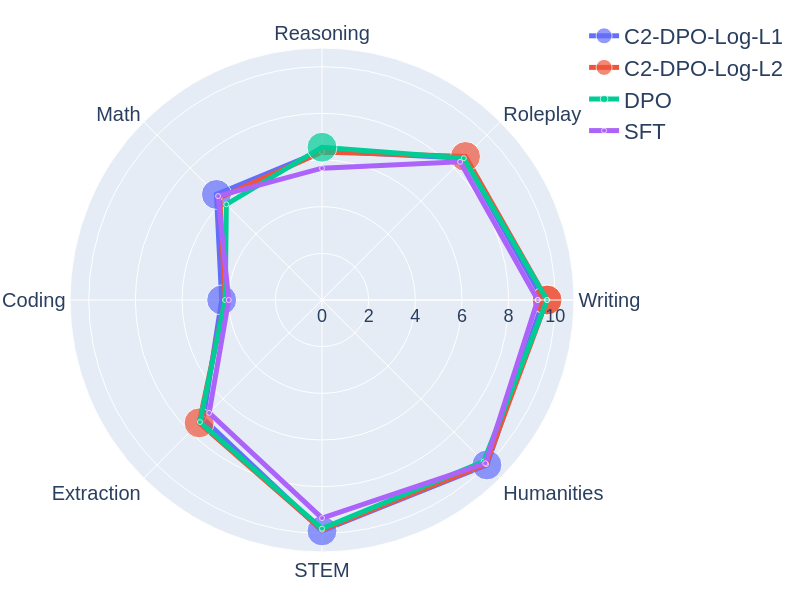} 
   \end{subfigure}
   \end{minipage}
   \caption{\textbf{Left}: A comparison between \texttt{C2-DPO} and \texttt{DPO} in terms of the probability of winner and loser responses during training. \textbf{Right}: Comparison between \texttt{C2-DPO} and \texttt{DPO} for aligning the Olmo-13B-SFT model using the Ultrafeedback-Binarized dataset.}
   \label{fig:training_losses} 
\end{figure}

We present the head to head win rate of \texttt{C2-DPO} against \texttt{DPO} by using 200 prompts from the held-out test set. We compute the win-rate by asking Anthropic's Claude Sonnet-3.5-v2 which response is more helpful. The exact prompt used for Claude is in Appendix~\ref{Appendix:win_rates_prompt}.


Recall from Section~\ref{sec:algos} that we proposed two candidates for the constraint function $\varphi$, namely the logarithmic and identity functions. Moreover, notice from~(\ref{eq:C-3DPO-log}) that we measure the deviation from the constraint using an $\ell_2$ penalty. Alternatively, we can measure this deviation using an $\ell_1$ penalty. Altogether, we have four specific implementations of \texttt{C2-DPO}: \texttt{C2-DPO-Log-$\ell_1$}, \texttt{C2-DPO-Log-$\ell_2$}, \texttt{C2-DPO-$\mathcal{I}$-$\ell_1$}, and \texttt{C2-DPO-$\mathcal{I}$-$\ell_2$}. In Appendix~\ref{Appendix:c-3dpo_pseudo_code}, we include pseudo-code for each of these 4 variations. In Figure~\ref{fig:four_side_by_side}, we show the head to head win rate of each of these 4 implementations against \texttt{DPO}, and in Figure~\ref{fig:uf_and_tldr} (Left), we show win-rate against the preferred response from the dataset. For all 4 implementations, we used the hyper-parameter $\lambda=2\times 10^{-4}$ and did not tune it for each of the 4 implementations separately. From this result, it is clear that all 4 variations improve upon DPO, with \texttt{C2-DPO-Log} being the highest performer.

Further, to holistically evaluate the final model, ~\citet{rasul2024preference} used MT-Bench, a multi-turn benchmark that uses GPT-4 to judge models’ performance in $8$ different categories: Writing, Roleplay, Reasoning, Math, Coding, Extraction, STEM, and Humanities. In Figure~\ref{fig:uf_7b_radar} we show the MT-Bench evaluation for \texttt{C2-DPO} against vanilla \texttt{DPO}, as well as related DPO-style algorithms that discuss the collapse of probabilities in \texttt{DPO} and aim to mitigate it, such as \texttt{Cal-DPO}~\citep{xiao2024caldpo}, \texttt{SPPO}~\citep{sppo}, and \texttt{DPOP}~\citep{smaug}. \texttt{C2-DPO} is the most competitive algorithm.

We then used \texttt{C2-DPO-Log} on a larger initial checkpoint, namely the \textrm{Olmo-13B-SFT} model from AllenAI~\cite{olmo20242olmo2furious}. We compared \texttt{C2-DPO-Log} against \texttt{DPO}. In Figure~\ref{fig:training_losses}, we see that \texttt{C2-DPO-Log} outperforms \texttt{DPO}, indicating that \texttt{C2-DPO} improvement may scale to larger models.

\subsection{\textrm{Reddit TL;DR}}
We then evaluate our proposed method on a summarization task with human-assigned scores for pairs of summaries. For this purpose, we employ the \textrm{Reddit TL;DR} dataset from \citet{reddit}. We follow \citet{odpo} in creating the dataset and include pairs of summaries where one received a higher quality score than the other. During training and testing, we select the highest-scoring summary as $y_w$ and a randomly selected alternative as $y_l$.

We align GPT-J \citet{gpt-j} with vanilla \texttt{DPO} as well as \texttt{C2-DPO} algorithms. Specifically, we first run one SFT epoch with the \textrm{Reddit TL;DR} dataset on GPT-J, and then perform subsequent preference alignment. We evaluate the final checkpoints by computing win rates against $y_w$ following \citet{DPO, odpo}. We use Claude Sonnet 3.5 v2 as a judge, and provide prompt used for Claude in the Appendix \ref{Appendix:win_rates_prompt_tldr}. As shown in Figure \ref{fig:training_losses}, \texttt{C2-DPO} improves upon \texttt{DPO}.  Moreover we show the probability of individual $y_w$ and $y_l$ in Figure~\ref{fig:training_losses}.

\section{Related Work}
\label{sec:related-work}

A few recent papers have proposed a unifying perspective on DPO-style algorithms. Notably,~\citet{tang2024generalized} presented a generalization of \texttt{DPO} where different supervised learning losses are applied to the difference of implicit rewards $\big(\hat r(x,y_w) - \hat r(x,y_l)\big)/\beta$.  
In contrast, we make a clear connection between DPO-style algorithms and classification, and also extend our results to lists and auxiliary information, as do similar efforts by~\citep{su2025reveal, zhao2025rainbow,im2024generalization,yao2025preference}. A second notable example was to generalize from KL to any f-divergence when measuring the discrepancy between the target and reference model~\citep{han2024f}. We also note that the first ingredient of our classification framework - $\pi_{\theta}$ - was used by~\citet{sharifnassab2024soft} to propose a soft version of \texttt{DPO}.

Earlier work studied the decrease of likelihood during training~\citep{feng2024towards, xie2024minor}, explained this phenomenon from different angles, and proposed different losses to address it. Here we provide a brief overview of a number of these results, especially those that we experimentally compare against. \citet{smaug} proposed \texttt{DPOP} which addresses the phenomenon by adding the penalty term $\max(0,-\hat r(x,y_w)/\beta)$ within the log-sigmoid of the \texttt{DPO} loss~\eqref{eq:DPO}. \texttt{DPOP} can also be viewed as \texttt{DPO} with a modified BT model. In this sense, it has similarities with the \texttt{$\alpha$-DPO}~\citep{AIPO} (see also \citep{choi2025self,shao2025earlier}).

\citet{xiao2024caldpo} attributed the undesirable behavior to the contrastive loss of \texttt{DPO} not being {\em scale-calibrated}, i.e.,~ignoring the absolute values of implicit rewards $\hat r(x,y_w)$ and $\hat r(x,y_l)$. They address this by constraining the implicit rewards to a scale that matches the ground-truth reward $r$. Thus, in their proposed loss, \texttt{Cal-DPO}, they add the square-loss $(\hat r(x,y) - r(x,y))^2,\;y\in\{y_w,y_l\}$ to the DPO loss (without $\beta$). Of course, since they do not have access to the ground-truth reward, they replace it with $1/2$ and $-1/2$ for $y_w$ and $y_l$, respectively. A loss similar to \texttt{Cal-DPO} was proposed by \citet{sppo} and they named it \texttt{SPPO}. It is simply \texttt{Cal-DPO} without the \texttt{DPO} loss.  Finally,~\citet{APO} proposed \texttt{APO}, which offer fine-grained control over the implicit rewards.
\section{Conclusion \& Future Work}
In this work, we revisited the derivation of DPO and revealed two counter-intuitive findings. We first proved that the standard DPO loss can arise from a formulation that regularizes only in-sample responses. This, in turn, uncovered another surprising result: both preferred and rejected responses are likely to experience a decrease in likelihood. This insight shed light on the likelihood-displacement phenomenon observed in prior empirical studies. We then generalized this result through a unifying classification framework, highlighting a broader absence of out-of-sample KL regularization.

Building on these findings, we then introduced Constrained Controlled DPO (\texttt{C2-DPO}), a principled extension of \texttt{DPO} that controls probability displacement by constraining the redistribution of likelihood between preferred and rejected responses. \texttt{C2-DPO} retains the computational efficiency of standard \texttt{DPO} while still offering a clear RLHF interpretation.

We know that \texttt{DPO} displaces probabilities to unseen samples, but we do not have a clear picture in terms of the kinds of responses for which the probability increases. While~\cite{distilled_DPO} and \cite{razin2025unintentional} present some findings for special cases, it would be interesting to do a more systematic study of the kinds of unseen responses whose likelihood increases during training.

\clearpage
\bibliographystyle{plainnat} 
\bibliography{references}

\begin{thebibliography}{40}
\providecommand{\natexlab}[1]{#1}
\providecommand{\url}[1]{\texttt{#1}}
\expandafter\ifx\csname urlstyle\endcsname\relax
  \providecommand{\doi}[1]{doi: #1}\else
  \providecommand{\doi}{doi: \begingroup \urlstyle{rm}\Url}\fi

\bibitem[Adler et~al.(2024)Adler, Agarwal, Aithal, Anh, Bhattacharya, Brundyn,
  Casper, Catanzaro, Clay, Cohen, et~al.]{nemotron}
Bo~Adler, Niket Agarwal, Ashwath Aithal, Dong~H Anh, Pallab Bhattacharya,
  Annika Brundyn, Jared Casper, Bryan Catanzaro, Sharon Clay, Jonathan Cohen,
  et~al.
\newblock Nemotron-4 340{B} technical report.
\newblock \emph{arXiv preprint arXiv:2406.11704}, 2024.

\bibitem[Amini et~al.(2024)Amini, Vieira, and Cotterell]{odpo}
Afra Amini, Tim Vieira, and Ryan Cotterell.
\newblock Direct preference optimization with an offset, 2024.
\newblock URL \url{https://arxiv.org/abs/2402.10571}.

\bibitem[Azar et~al.(2024)Azar, Guo, Piot, Munos, Rowland, Valko, and
  Calandriello]{IPO}
Mohammad~Gheshlaghi Azar, Zhaohan~Daniel Guo, Bilal Piot, Remi Munos, Mark
  Rowland, Michal Valko, and Daniele Calandriello.
\newblock A general theoretical paradigm to understand learning from human
  preferences.
\newblock In \emph{International Conference on Artificial Intelligence and
  Statistics}, pages 4447--4455. PMLR, 2024.

\bibitem[Choi et~al.(2025)Choi, Ahmadian, Geist, Pietquin, and
  Azar]{choi2025self}
Eugene Choi, Arash Ahmadian, Matthieu Geist, Oilvier Pietquin, and
  Mohammad~Gheshlaghi Azar.
\newblock Self-improving robust preference optimization, 2025.

\bibitem[Christiano et~al.(2017)Christiano, Leike, Brown, Martic, Legg, and
  Amodei]{RLHF}
Paul~F Christiano, Jan Leike, Tom Brown, Miljan Martic, Shane Legg, and Dario
  Amodei.
\newblock Deep reinforcement learning from human preferences.
\newblock \emph{Advances in neural information processing systems}, 30, 2017.

\bibitem[Deng et~al.(2025)Deng, Ha, Rui, Fuli, Zheng, and Xiangnan]{dpoless}
Xun Deng, Zhong Ha, Ai~Rui, Feng Fuli, Wang Zheng, and He~Xiangnan.
\newblock Less is more: Improving llm alignment via preference data selection.
\newblock \emph{arXiv preprint arXiv:2502.14560}, 2025.

\bibitem[D'Oosterlinck et~al.(2024)D'Oosterlinck, Xu, Develder, Demeester,
  Singh, Potts, Kiela, and Mehri]{APO}
Karel D'Oosterlinck, Winnie Xu, Chris Develder, Thomas Demeester, Amanpreet
  Singh, Christopher Potts, Douwe Kiela, and Shikib Mehri.
\newblock Anchored preference optimization and contrastive revisions:
  Addressing underspecification in alignment.
\newblock \emph{arXiv preprint arXiv:2408.06266}, 2024.

\bibitem[Feng et~al.(2024)Feng, Qin, Huang, Zhang, and Lei]{feng2024towards}
Duanyu Feng, Bowen Qin, Chen Huang, Zheng Zhang, and Wenqiang Lei.
\newblock Towards analyzing and understanding the limitations of dpo: A
  theoretical perspective.
\newblock \emph{arXiv preprint arXiv:2404.04626}, 2024.

\bibitem[Fisch et~al.(2024)Fisch, Eisenstein, Zayats, Agarwal, Beirami, Nagpal,
  Shaw, and Berant]{distilled_DPO}
Adam Fisch, Jacob Eisenstein, Vicky Zayats, Alekh Agarwal, Ahmad Beirami,
  Chirag Nagpal, Pete Shaw, and Jonathan Berant.
\newblock Robust preference optimization through reward model distillation.
\newblock \emph{arXiv preprint arXiv:2405.19316}, 2024.

\bibitem[Guo et~al.(2024)Guo, Lu, Bo, and Jiaqi]{dpotodo}
Yuxiang Guo, Yin Lu, Jiang Bo, and Zhang Jiaqi.
\newblock {TODO}: Enhancing llm alignment with ternary preferences.
\newblock \emph{arXiv preprint arXiv:2411.02442}, 2024.

\bibitem[Han et~al.(2024)Han, Jiang, Song, Leskovec, Ermon, and Xu]{han2024f}
Jiaqi Han, Mingjian Jiang, Yuxuan Song, Jure Leskovec, Stefano Ermon, and
  Minkai Xu.
\newblock $ f $-po: Generalizing preference optimization with $ f $-divergence
  minimization.
\newblock \emph{arXiv preprint arXiv:2410.21662}, 2024.

\bibitem[Huang et~al.(2025)Huang, Zhan, Xie, Lee, Sun, Krishnamurthy, and
  Foster]{huang2025correcting}
Audrey Huang, Wenhao Zhan, Tengyang Xie, Jason~D. Lee, Wen Sun, Akshay
  Krishnamurthy, and Dylan~J. Foster.
\newblock Correcting the mythos of kl-regularization: Direct alignment without
  overoptimization via chi-squared preference optimization, 2025.

\bibitem[Im and Li(2024)]{im2024generalization}
Shawn Im and Yixuan Li.
\newblock On the generalization of preference learning with dpo.
\newblock \emph{arXiv preprint arXiv:2408.03459}, 2024.

\bibitem[Liu et~al.(2024)Liu, Lu, Zhang, Liu, Guo, Yang, Blanchet, and
  Wang]{NEURIPS2024_fa69e968}
Zhihan Liu, Miao Lu, Shenao Zhang, Boyi Liu, Hongyi Guo, Yingxiang Yang, Jose
  Blanchet, and Zhaoran Wang.
\newblock Provably mitigating overoptimization in rlhf: Your sft loss is
  implicitly an adversarial regularizer.
\newblock In A.~Globerson, L.~Mackey, D.~Belgrave, A.~Fan, U.~Paquet,
  J.~Tomczak, and C.~Zhang, editors, \emph{Advances in Neural Information
  Processing Systems}, volume~37, pages 138663--138697. Curran Associates,
  Inc., 2024.

\bibitem[Mitchell(2023)]{mitchell2023dpo_implementation}
Eric Mitchell.
\newblock Direct preference optimization.
\newblock github.com/eric-mitchell/direct-preference-optimization, 2023.
\newblock Accessed: 2024-1-01.

\bibitem[Mitchell(2024)]{CDPO}
Eric Mitchell.
\newblock A note on {DPO} with noisy preferences and relationship to {IPO},
  2024.
\newblock URL \url{https://ericmitchell.ai/cdpo.pdf}.

\bibitem[M{\"u}ller et~al.(2019)M{\"u}ller, Kornblith, and Hinton]{softLabels}
Rafael M{\"u}ller, Simon Kornblith, and Geoffrey~E Hinton.
\newblock When does label smoothing help?
\newblock \emph{Advances in neural information processing systems}, 32, 2019.

\bibitem[OLMo et~al.(2024)OLMo, Walsh, Soldaini, Groeneveld, Lo, Arora, Bhagia,
  Gu, Huang, Jordan, Lambert, Schwenk, Tafjord, Anderson, Atkinson, Brahman,
  Clark, Dasigi, Dziri, Guerquin, Ivison, Koh, Liu, Malik, Merrill, Miranda,
  Morrison, Murray, Nam, Pyatkin, Rangapur, Schmitz, Skjonsberg, Wadden,
  Wilhelm, Wilson, Zettlemoyer, Farhadi, Smith, and
  Hajishirzi]{olmo20242olmo2furious}
Team OLMo, Pete Walsh, Luca Soldaini, Dirk Groeneveld, Kyle Lo, Shane Arora,
  Akshita Bhagia, Yuling Gu, Shengyi Huang, Matt Jordan, Nathan Lambert, Dustin
  Schwenk, Oyvind Tafjord, Taira Anderson, David Atkinson, Faeze Brahman,
  Christopher Clark, Pradeep Dasigi, Nouha Dziri, Michal Guerquin, Hamish
  Ivison, Pang~Wei Koh, Jiacheng Liu, Saumya Malik, William Merrill, Lester
  James~V. Miranda, Jacob Morrison, Tyler Murray, Crystal Nam, Valentina
  Pyatkin, Aman Rangapur, Michael Schmitz, Sam Skjonsberg, David Wadden,
  Christopher Wilhelm, Michael Wilson, Luke Zettlemoyer, Ali Farhadi, Noah~A.
  Smith, and Hannaneh Hajishirzi.
\newblock 2 olmo 2 furious.
\newblock 2024.
\newblock URL \url{https://arxiv.org/abs/2501.00656}.

\bibitem[Pal et~al.(2024)Pal, Karkhanis, Dooley, Roberts, Naidu, and
  White]{smaug}
Arka Pal, Deep Karkhanis, Samuel Dooley, Manley Roberts, Siddartha Naidu, and
  Colin White.
\newblock Smaug: Fixing failure modes of preference optimisation with
  {DPO}-positive.
\newblock \emph{arXiv preprint arXiv:2402.13228}, 2024.

\bibitem[Rafailov et~al.(2023)Rafailov, Sharma, Mitchell, Manning, Ermon, and
  Finn]{DPO}
Rafael Rafailov, Archit Sharma, Eric Mitchell, Christopher~D Manning, Stefano
  Ermon, and Chelsea Finn.
\newblock Direct preference optimization: Your language model is secretly a
  reward model.
\newblock In \emph{Advances in Neural Information Processing Systems},
  volume~36, pages 53728--53741, 2023.

\bibitem[Rasul et~al.(2024)Rasul, Beeching, Tunstall, von Werra, and
  Sanseviero]{rasul2024preference}
Kashif Rasul, Edward Beeching, Lewis Tunstall, Leandro von Werra, and Omar
  Sanseviero.
\newblock Preference tuning llms with direct preference optimization methods,
  2024.
\newblock URL \url{https://huggingface.co/blog/pref-tuning}.

\bibitem[Razin et~al.(2025)Razin, Malladi, Bhaskar, Chen, Arora, and
  Hanin]{razin2025unintentional}
Noam Razin, Sadhika Malladi, Adithya Bhaskar, Danqi Chen, Sanjeev Arora, and
  Boris Hanin.
\newblock Unintentional unalignment: Likelihood displacement in direct
  preference optimization.
\newblock In \emph{International Conference on Learning Representations}, 2025.

\bibitem[Shao et~al.(2025)Shao, Li, Liu, Chen, Zhou, Wang, Cai, and
  Li]{shao2025earlier}
Ruichen Shao, Bei Li, Gangao Liu, Yang Chen, Xiang Zhou, Jingang Wang, Xunliang
  Cai, and Peng Li.
\newblock Earlier tokens contribute more: Learning direct preference
  optimization from temporal decay perspective, 2025.

\bibitem[Sharifnassab et~al.(2024)Sharifnassab, Salehkaleybar, Ghiassian,
  Kanoria, and Schuurmans]{sharifnassab2024soft}
Arsalan Sharifnassab, Saber Salehkaleybar, Sina Ghiassian, Surya Kanoria, and
  Dale Schuurmans.
\newblock Soft preference optimization: Aligning language models to expert
  distributions.
\newblock \emph{arXiv preprint arXiv:2405.00747}, 2024.

\bibitem[Shen et~al.(2024)Shen, Wang, Niu, Zhou, Tang, Zhang, Chen, and
  Wen]{AIPO}
Yaojie Shen, Xinyao Wang, Yulei Niu, Ying Zhou, Lexin Tang, Libo Zhang, Fan
  Chen, and Longyin Wen.
\newblock {AIPO}: Improving training objective for iterative preference
  optimization.
\newblock \emph{arXiv preprint arXiv:2409.08845}, 2024.

\bibitem[Stiennon et~al.(2022)Stiennon, Ouyang, Wu, Ziegler, Lowe, Voss,
  Radford, Amodei, and Christiano]{reddit}
Nisan Stiennon, Long Ouyang, Jeff Wu, Daniel~M. Ziegler, Ryan Lowe, Chelsea
  Voss, Alec Radford, Dario Amodei, and Paul Christiano.
\newblock Learning to summarize from human feedback, 2022.
\newblock URL \url{https://arxiv.org/abs/2009.01325}.

\bibitem[Su et~al.(2025)Su, Wang, Zhu, Yi, Xu, Ma, and Liu]{su2025reveal}
Xuerui Su, Yue Wang, Jinhua Zhu, Mingyang Yi, Feng Xu, Zhiming Ma, and Yuting
  Liu.
\newblock Reveal the mystery of dpo: The connection between dpo and rl
  algorithms.
\newblock \emph{arXiv preprint arXiv:2502.03095}, 2025.

\bibitem[Tang et~al.(2024)Tang, Guo, Zheng, Calandriello, Munos, Rowland,
  Richemond, Valko, Pires, and Piot]{tang2024generalized}
Yunhao Tang, Zhaohan~Daniel Guo, Zeyu Zheng, Daniele Calandriello, R{\'e}mi
  Munos, Mark Rowland, Pierre~Harvey Richemond, Michal Valko,
  Bernardo~{\'A}vila Pires, and Bilal Piot.
\newblock Generalized preference optimization: A unified approach to offline
  alignment.
\newblock \emph{arXiv preprint arXiv:2402.05749}, 2024.

\bibitem[Wang and Komatsuzaki(2021)]{gpt-j}
Ben Wang and Aran Komatsuzaki.
\newblock {GPT-J-6B: A 6 Billion Parameter Autoregressive Language Model}.
\newblock \url{https://github.com/kingoflolz/mesh-transformer-jax}, May 2021.

\bibitem[Wu et~al.(2024)Wu, Sun, Yuan, Ji, Yang, and Gu]{sppo}
Yue Wu, Zhiqing Sun, Huizhuo Yuan, Kaixuan Ji, Yiming Yang, and Quanquan Gu.
\newblock Self-play preference optimization for language model alignment,
  2024b.
\newblock \emph{arXiv preprint arXiv:2405.00675}, 2024.

\bibitem[Xiao et~al.(2024)Xiao, Yuan, Zhu, Li, and Honavar]{xiao2024caldpo}
Teng Xiao, Yige Yuan, Huaisheng Zhu, Mingxiao Li, and Vasant~G Honavar.
\newblock Cal-{DPO}: Calibrated direct preference optimization for language
  model alignment.
\newblock In \emph{The Thirty-eighth Annual Conference on Neural Information
  Processing Systems}, 2024.

\bibitem[Xie et~al.(2024)Xie, Chen, Yu, Sun, Wu, and Hu]{xie2024minor}
Shiming Xie, Hong Chen, Fred Yu, Zeye Sun, Xiuyu Wu, and Yingfan Hu.
\newblock Minor dpo reject penalty to increase training robustness.
\newblock \emph{arXiv preprint arXiv:2408.09834}, 2024.

\bibitem[Xiliang et~al.(2025)Xiliang, Feng, Qianen, Lei, and Xiao]{dposhift}
Yang Xiliang, Jiang Feng, Zhang Qianen, Zhao Lei, and Li~Xiao.
\newblock {DPO}-shift: Shifting the distribution of direct preference
  optimization.
\newblock \emph{arXiv preprint arXiv:2502.07599v1}, 2025.

\bibitem[Xu et~al.(2024)Xu, Fu, Gao, Ye, Liu, Mei, Wang, Yu, and
  Wu]{ppo_vs_dpo}
Shusheng Xu, Wei Fu, Jiaxuan Gao, Wenjie Ye, Weilin Liu, Zhiyu Mei, Guangju
  Wang, Chao Yu, and Yi~Wu.
\newblock Is {DPO} superior to {PPO} for {LLM} alignment? a comprehensive
  study.
\newblock \emph{arXiv preprint arXiv:2404.10719}, 2024.

\bibitem[Yao et~al.(2025)Yao, Cai, Chuang, Yang, Jiang, Yang, and
  Hu]{yao2025preference}
Binwei Yao, Zefan Cai, Yun-Shiuan Chuang, Shanglin Yang, Ming Jiang, Diyi Yang,
  and Junjie Hu.
\newblock No preference left behind: Group distributional preference
  optimization, 2025.

\bibitem[Yin et~al.(2024)Yin, Chak~Tou, Hongbo, Minjun, Hanqi, Qiang, Yulan,
  Wenjie, Jun, Yue, and Linyi]{dposparse}
Qingyu Yin, Leong Chak~Tou, Zhang Hongbo, Zhu Minjun, Yan Hanqi, Zhang Qiang,
  He~Yulan, Li~Wenjie, Wang Jun, Zhang Yue, and Yang Linyi.
\newblock Direct preference optimization using sparse feature-level
  constraints.
\newblock \emph{arXiv preprint arXiv:2411.07618}, 2024.

\bibitem[Yunan et~al.(2025)Yunan, Jijie, Bo-Wen, Liangdong, and
  Guang]{dpobalancing}
Wang Yunan, Li~Jijie, Zhang Bo-Wen, Wang Liangdong, and Liu Guang.
\newblock Inco-{DPO}: Balancing distribution shift and data quality for
  enhancedpreference optimization.
\newblock \emph{arXiv preprint arXiv:2503.15880v1}, 2025.

\bibitem[Yuzi et~al.(2025)Yuzi, Yibo, Jialian, Yipin, Jian, Zhijie, and
  Dong]{yuzi2025identifying}
Yan Yuzi, Miao Yibo, Li~Jialian, Zhang Yipin, Xie Jian, Deng Zhijie, and Yan
  Dong.
\newblock 3{D}-properties: Identifying challenges in dpo and charting a path
  forward.
\newblock In \emph{International Conference on Learning Representations}, 2025.

\bibitem[Zhao et~al.(2025)Zhao, Winata, Das, Zhang, Yao, Tang, and
  Sahu]{zhao2025rainbow}
Hanyang Zhao, Genta~Indra Winata, Anirban Das, Shi-Xiong Zhang, David~D. Yao,
  Wenpin Tang, and Sambit Sahu.
\newblock Rainbowpo: A unified framework for combining improvements in
  preference optimization, 2025.

\bibitem[Zheng et~al.(2023)Zheng, Chiang, Sheng, Zhuang, Wu, Zhuang, Lin, Li,
  Li, Xing, et~al.]{mt_bench}
Lianmin Zheng, Wei-Lin Chiang, Ying Sheng, Siyuan Zhuang, Zhanghao Wu, Yonghao
  Zhuang, Zi~Lin, Zhuohan Li, Dacheng Li, Eric Xing, et~al.
\newblock Judging llm-as-a-judge with mt-bench and chatbot arena.
\newblock \emph{Advances in Neural Information Processing Systems},
  36:\penalty0 46595--46623, 2023.

\end{thebibliography}

\clearpage
\onecolumn
\appendix
\section{Proof of Lemmas \ref{L:DPO_non_standard} and \ref{lemma: wierd e^-1 phenomenon}}
\label{Appendix_Idan}
    In order to prove Lemma \ref{L:DPO_non_standard} we will need first the following technical result. We recall that $\Delta_{n}$ denotes the simplex in $\mathbb{R}^{n}$ meaning the set of all vectors $p \in \mathbb{R}^{n}$ for which $\sum_{i = 1}^{n} p_{i} = 1$ and $p \geq 0$.

\begin{lemma}\label{lemma: KL of subset S}
    Let $q \in \Delta_{n}$ and $r \in \mathbb{R}^{n}$. Let $S \subseteq [n] := \{ 1 , 2 , \ldots , n \}$ be a subset for which $|S| \geq 2$. Any optimal solution $p$ of the following problem
    \begin{equation}\label{eq: S problem}
        \argmax_{p \in \Delta_{n}} \left\{ r^{T}p - \beta \sum_{i \in S} p_i \ln \frac{p_i}{q_i} \right\} \; .
    \end{equation}
    satisfies that $\beta\left(\ln\left(p_{i}/q_{i}\right) - \ln\left(p_{j}/q_{j}\right)\right) = r_{i} - r_{j}$ holds true for any $i , j \in S$.
\end{lemma}
\begin{proof}
    The problem \eqref{eq: S problem} is convex and Slater's condition is satisfied; hence, the set of optimal solutions coincides with the set of KKT points. Therefore, we will consider the Lagrangian
    \begin{equation}
        L(p, \lambda, \mu) = r^{T}p - \beta \sum_{i \in S} p_i \ln \frac{p_i}{q_i} + \lambda^T p + \mu\left(\sum_{i=1}^n p_i - 1\right) \; ,
    \end{equation}
    where $\lambda \in \mathbb{R}_{+}^{n}$ and $\mu \in \mathbb{R}$ are the Lagrange multipliers. The KKT conditions are
    \begin{align}
    \begin{cases}
         \frac{\partial L}{\partial p_i} = r_i - \beta \left( \ln \frac{p_i}{q_i} + 1 \right) + \lambda_i + \mu = 0, & \forall i \in S\\
        \frac{\partial L}{\partial p_i} = r_i + \lambda_i + \mu= 0, & \forall i \notin S, \\
        p_i \geq 0, \quad \sum_{i=1}^{n} p_i = 1, & \text{(feasibility)}, \\
        \lambda_i \geq 0, & \forall i \in [n], \quad \text{(multipliers)}, \\
        \lambda_i p_i = 0, & \forall i \in [n], \quad \text{(complementary slackness)}.        
    \end{cases}
    \end{align}
    Let $i \in S$, after rearranging the condition we get that $p_i = q_i e^{(r_i + \lambda_i + \mu)/\beta - 1}$. Moreover, if $q_i > 0$ then $p_i > 0$ and hence $\lambda_i = 0$. On the other hand, if $q_i = 0$ then $p_i = 0$ and the value of $\lambda_i$ does not matter. Hence, for simplicity, we can take $\lambda_{i} = 0$ for all $i \in S$. Thus, the KKT conditions can be equivalently written as follows
    \begin{align}
    \begin{cases}
        p_i = q_i e^{(r_i + \mu)/\beta - 1}, & \forall i \in S. \\
        r_i + \lambda_i + \mu= 0, & \forall i \notin S, \\
        p \in \Delta_{n}, & \text{(feasibility)}, \\
        \lambda_i \geq 0, & \forall i \notin S, \quad \text{(multipliers)}, \\
        \lambda_i p_i = 0, & \forall i \notin S, \quad \text{(complementary slackness)}.
    \end{cases}
    \end{align}    
    Let $\hat{r} = \max_{i \notin S} r_i$. We now split the proof into the two cases:
    \begin{itemize}
        \item \underline{Case I: $\hat{r} > \beta \ln \sum_{i \in S} q_i e^{r_i /\beta - 1}$}

        First, for any $i \notin S$, since $\lambda_i \geq 0$ we have that $\mu \leq -r_i$ and therefore in particular $\mu \leq -\hat{r}$. Thus, there must be $i \notin S$ for which $p_i > 0$. Indeed, if this is not the case then $p_{i} = 0$ for all $i \notin S$ and we get a contradiction since (recall that $\mu \leq -\hat{r}$)
        \begin{equation} 
            1 = \sum_{i = 1}^{n} p_i = \sum_{i \in S} p_i = \sum_{i \in S} q_i e^{(r_i + \mu)/\beta - 1} \leq \sum_{i \in S} q_i e^{(r_i - \hat{r})/\beta - 1} < 1 \ ,
        \end{equation}
        where the last inequality follows from the condition of Case I. Therefore, for simplicity we take $i_{\ast} \notin S$ for which $p_{i_{\ast}} > 0$. Therefore, $\lambda_{i_{\ast}} = 0$.
        
        Now, for all $j \notin S$, we have $r_{i_{\ast}} + \lambda_{i_{\ast}} = r_j + \lambda_j$ and thus $r_{i_{\ast}} = r_j + \lambda_j \geq r_j$, which means that $r_{i_{\ast}} = \hat{r}$. Hence $\mu = - \hat{r}$, and therefore $p_i = q_i e^{(r_i - \hat{r})/\beta - 1}$ for all $i \in S$.
        
        Moreover, this shows that if for some $i \notin S$ we have $r_i < \hat{r}$, then $p_i = 0$. Therefore, by using the feasibility condition we obtain that any KKT point can be described by
        \begin{align*}
            p_i & = q_i e^{(r_i - \hat{r})/\beta - 1},  \quad \forall i \in S \\ 
            p_i & = 0, \quad \forall i \notin S, r_i < \hat{r} \\
            \sum_{j \notin S, r_j = \hat{r}} p_j & = 1 - \sum_{j \in S} q_j e^{(r_j - \hat{r})/\beta -1}  \; .
        \end{align*}
        \item \underline{Case II: $\hat{r} \le \beta \ln \sum_{i \in S} q_i e^{r_i/\beta - 1 }$} 
    
        In this case, we will first prove that $p_i = 0$ for all $i \notin S$. Suppose in contradiction that $p_{\ell} > 0$ for some $\ell \notin S$. Then, $\lambda_{\ell} = 0$ and hence $\mu = -r_{\ell}$. Recall that $r_\ell \le \hat{r}$, we obtain that
        \begin{equation*}
            1 = \sum_{i = 1}^{n} p_i \geq p_{\ell} + \sum_{i \in S} p_i = p_\ell + \sum_{i \in S} q_i e^{(r_i - r_\ell)/\beta - 1} \ge p_\ell + \sum_{i \in S} q_i e^{(r_i - \hat{r})/\beta - 1} \ge p_\ell + 1 > 1 \; ,
        \end{equation*}
        where the third inequality follows from the condition of Case II. Therefore, the unique KKT point is given by
        \begin{align*}            
            p_i & = \frac{q_i e^{r_i/\beta}}{ \sum_{j \in S} q_j e^{r_j/\beta}}, \quad \forall i \in S \\
            p_i & = 0,   \quad \forall i \notin S \; .
        \end{align*}
    \end{itemize}
    Finally, we see that in both cases, we have, for any $i , j \in S$, that
    \begin{equation*}
        \beta\left(\ln\left(p_{i}/q_{i}\right) - \ln\left(p_{j}/q_{j}\right)\right) = r_{i} - r_{j},
    \end{equation*}
    as required.
\end{proof}
Now, we can prove Lemma \ref{L:DPO_non_standard}. To this end, we recall the new optimization problem that we study
\begin{equation} \label{eq:non-standard_RLHF-app}
    \!\!\!\!\max_{\theta}\ \E{x}\left[ \E{y \sim \pi_{\theta}}\big[r_\phi(x,y)\big] - \beta \sum_{y \in S_{x}} \pi_{\theta}(y | x)\log \frac{\pi_{\theta}(y | x)}{\piRef(y | x)} \right].
\end{equation}

\textit{Proof of Lemma \ref{L:DPO_non_standard}.}
    From Lemma \ref{lemma: KL of subset S}, the optimization problem \eqref{eq:non-standard_RLHF-app} admits a closed-form solution
    

    $\pi_\theta$, such that for any two responses $y, y'$ within the set $S_x$ satisfy
    \begin{equation}
        r_\phi(x, y) - r_\phi(x, y') = \beta \ln{\frac{\pi_\theta(y | x)}{\piRef(y | x)}} - \beta \ln \frac{\pi_\theta(y' | x)}{\piRef(y' | x)} \; .
    \end{equation}

    Therefore, in particular, for any triplet in the preference dataset $(x, y_w, y_l) \in \mathcal{D}$, the BT model gives us

    \begin{equation}\label{eq: probabiity in terms of r in proof}
        p(y_w \succ y_l | x) = \sigma\left(\beta \ln{\frac{\pi_\theta(y_w | x)}{\piRef(y_w | x)}} - \beta \ln \frac{\pi_\theta(y_l | x)}{\piRef(y_l | x)} \right) \; ,
    \end{equation}

    which are exactly the probabilities used in the \texttt{DPO} derivation, and taking the negative-log-likelihood of \eqref{eq: probabiity in terms of r in proof} over the entire dataset yields exactly the \texttt{DPO} loss. \qed

    We complete this section with the proof of Lemma \ref{lemma: wierd e^-1 phenomenon}.

    \textit{Proof of Lemma \ref{lemma: wierd e^-1 phenomenon}.}
    If $r_\phi(x,y) \le \hat{r}_x$ for every $y \in S_x$, than the condition $\hat{r}_x > \beta \ln(\sum_{y \in S_x} \piRef(y | x)e^{r_{\phi}(x,y)/\beta - 1})$ is satisfied. This is because
    \begin{align*}
        \beta \ln\sum_{y \in S_x} \piRef(y | x) e^{r_\phi(x,y)/\beta - 1} & = \beta \ln\sum_{y \in S_x} \piRef(y | x)e^{\hat{r}_x/\beta + (r_\phi(x,y) - \hat{r}_x)/\beta - 1} \\ 
        & = \hat{r}_x + \beta \ln\sum_{y \in S_x} \piRef(y | x) e^{(r_\phi(x,y) - \hat{r}_x)/\beta - 1} \\
        & \le \hat{r}_x + \beta \ln\sum_{y \in S_x} \piRef(y | x) e^{-1} \\ 
        & = \hat{r}_x + \beta\left(-1 + \ln \sum_{y \in S_x} \piRef(y | x)\right) \\ 
        & \leq \hat{r}_x - \beta \ ,
    \end{align*}
    where the last inequality follows from the fact that $\sum_{y \in S_x} \piRef(y | x) \leq 1$. This proves that the condition holds true.
    
    Therefore, as we saw in the proof of Lemma \ref{L:DPO_non_standard}, in any optimal solution $\pi_\theta$ we have for all $y \in S_x$, that
    \begin{equation*}
        \pi_\theta(y | x) = \piRef(y | x) e^{(r_\phi(x, y) - \hat{r}_x)/\beta - 1} \le \piRef(y | x) e^{-1} = e^{-1}\piRef(y | x),
    \end{equation*}
    which proves the desired result. \qed

\section{Derivations of Existing Algorithms via the Classification Framework} \label{Appendix:Algo}
\subsection{\texttt{IPO}} \label{Appendix:IPO}

    We recall that \texttt{IPO} can be formulated as a classification problem with the soft labels $p \assign (p^{w}, p^{l}) = (\sigma(1/2) , \sigma(-1/2))$ and the loss given by
    \begin{equation} \label{IPO-Take2}
        \mathcal{L}\big(p_{\theta},\ p\big) = \left(\log\frac{p^{w}_{\theta}}{p^{l}_{\theta}} - \log\frac{p^{w}}{p^{l}}\right)^2\ .
    \end{equation}
    To show that we indeed recover the \texttt{IPO}, we first note that
    \begin{equation*}
        \frac{p^{w}}{p^{l}} = \frac{\sigma(1/2)}{\sigma(-1/2)} = \frac{1 + \exp(1/2)}{1 + \exp(-1/2)} = \exp(1/2)\ ,
    \end{equation*}
    where the second equality follows from the definition of the sigmoid $\sigma(x) = 1/(1 + \exp(-x))$. Moreover, using the definitions of $p^{w}_{\theta}$ (see \eqref{pthetaw}) and $p^{l}_{\theta}$, we obtain that 
    \begin{equation} \label{IPOfrac}
        \frac{p^{w}_{\theta}}{p^{l}_{\theta}} = \frac{\big(\frac{\pi_{\theta}(y_w \mid x)}{\piRef(y_w \mid x)}\big)^{\beta}}{\big(\frac{\pi_{\theta}(y_l \mid x)}{\piRef(y_l \mid x)}\big)^{\beta}} = \left(\frac{\pi_{\theta}(y_w | x)\piRef(y_l | x)}{\pi_{\theta}(y_l | x)\piRef(y_w | x)}\right)^{\beta}\ .
    \end{equation}
    Plugging these two developments to the loss in \eqref{IPO-Take2} yields
    \begin{equation*}
        \mathcal{L}\big(p_{\theta},\ p\big) = \left(\log\frac{p^{w}_{\theta}}{p^{l}_{\theta}} - \log\frac{p^{w}}{p^{l}}\right)^2 = \left(\beta\log \frac{\pi_{\theta}(y_w | x)\piRef(y_l | x)}{\pi_{\theta}(y_l | x)\piRef(y_w | x)} - \frac{1}{2}\right)^2 = \left(\log \frac{\pi_{\theta}(y_w | x)\piRef(y_l | x)}{\pi_{\theta}(y_l | x)\piRef(y_w | x)} - \frac{1}{2\beta}\right)^2\ ,
    \end{equation*} 
    which is exactly \texttt{IPO} (see Eq. (17) of~\citealt{IPO}).

\subsection{\texttt{CDPO}} \label{Appendix:CDPO}
    Recall the \texttt{CDPO} \cite{CDPO} loss is given for any triplet $(x , y_{w} , y_{l})$ by
    \begin{equation*}
        -(1 - \varepsilon)\log \sigma\left(\beta\log \frac{\pi_{\theta}(y_w | x)}{\piRef(y_w | x)} - \beta\log \frac{\pi_{\theta}(y_l | x)}{\piRef(y_l | x)}\right) - \varepsilon\log \sigma\left(\beta\log \frac{\pi_{\theta}(y_l | x)}{\piRef(y_l | x)} - \beta\log \frac{\pi_{\theta}(y_w | x)}{\piRef(y_w | x)}\right)\ ,
    \end{equation*}
    and then summed over all the triplets in the dataset ${\cal D}$. In order to see \texttt{CDPO} as a classification with the soft labels $p := (p^w,p^l) = (1-\varepsilon,\varepsilon)$ and the CE loss, we will use the following technical fact
    \begin{equation*}
        \sigma(\beta\log(a/b)) = \frac{1}{1 + \exp(-\beta\log(a/b))} = \frac{1}{1 + \frac{\exp{\log b^{\beta}}}{\exp{\log a^{\beta}}}} = \frac{1}{1 + \frac{b^{\beta}}{a^{\beta}}} = \frac{a^{\beta}}{a^{\beta} + b^{\beta}}.
    \end{equation*}
    Therefore, with $a = \frac{\pi_{\theta}(y_w|x)}{\piRef(y_w|x)}$ and $b = \frac{\pi_{\theta}(y_l|x)}{\piRef(y_l|x)}$ we get from \eqref{pthetaw} that $\sigma(\beta\log a/b) = p^{w}_{\theta}$. Similarly, we get that $\sigma(\beta\log(b/a)) = p^{l}_{\theta}$. Therefore, the \texttt{CDPO} loss can be written as follows
    \begin{equation*}
         -(1 - \varepsilon)\log p^{w}_{\theta} - \varepsilon \log p^{l}_{\theta} = -p^{w}\log p^{w}_{\theta} - p^{l}\log p^{l}_{\theta}\ ,
    \end{equation*}
    which is exactly the CE loss on the vectors $p_{\theta} = (p^{w}_{\theta} , p^{l}_{\theta})$ and $p = (p^w,p^l)$.

\subsection{\texttt{DPO (PL)}} \label{Appendix:DPO_PL}
    In this setting, we are given a dataset of the form $\mathcal{D} = \{ (x,y_1, y_2,\ldots,y_N) \}$. In order to recover, the \texttt{DPO} with Plackett-Luce~\cite{DPO}, we need to generalize the definition of the probability vector $p_{\theta}$ from pairs as in \eqref{pthetaw} to the following $N - 1$ subsets of the list, namely the first $N$, then the first $N-1$, then first $N-2$ and so on. More precisely, for any $1 \leq n < N$ we define 
    \begin{equation*}
        p_{\theta}(x,y_n, y_{n + 1} , \ldots, y_N) = \textrm{softmax}\big(\left(r_{\theta}(x,y_n), r_{\theta}(x,y_{n+1}), \ldots, r_{\theta}(x,y_N)\right)\big)\ .
    \end{equation*}
    In this case, the hard label vectors are defined, for any $1 \leq n < N$, by $p^{[n,N]} := (1 , 0 , \ldots , 0) \in \mathbb{R}^{N - n + 1}$. Now, using the CE loss we get the desired result as follows
    \begin{eqnarray*}
        -\sum_{n = 1}^{N - 1}\sum_{i = n}^{N} p_{i}^{[n,N]}\log p_{\theta}^{i}(x,y_n, y_{n + 1} , \ldots, y_N) & = & - \sum_{n = 1}^{N - 1} \log \frac{\exp\left(\beta\log \frac{Z(x)\pi_{\theta}(y_{n}|x)}{\piRef(y_{n}|x)}\right)}{\sum_{i = n}^{N} \exp\left(\beta\log \frac{Z(x)\pi_{\theta}(y_{i}|x)}{\piRef(y_{i}|x)}\right)}\\
        & = & - \sum_{n = 1}^{N - 1} \log \frac{\exp\left(\beta\log \frac{\pi_{\theta}(y_{n}|x)}{\piRef(y_{n}|x)}\right)}{\sum_{i = n}^{N} \exp\left(\beta\log \frac{\pi_{\theta}(y_{i}|x)}{\piRef(y_{i}|x)}\right)}\\
        & = & - \log \prod_{n = 1}^{N - 1} \frac{\exp\left(\beta\log \frac{\pi_{\theta}(y_{n}|x)}{\piRef(y_n|x)}\right)}{\sum_{i = n}^{N} \exp\left(\beta\log \frac{\pi_{\theta}(y_{i}|x)}{\piRef(y_{i}|x)}\right)}\ .
    \end{eqnarray*}

\subsection{\texttt{RPO} and \texttt{Distilled DPO}} \label{Appendix:RPO}
    As we discussed in Section 3.2, \texttt{RPO} can be reformulated as a classification with soft labels, which are defined by $p = \text{softmax}(s_w, s_l)$. Therefore, we immediately see that
    \begin{equation*}
        p^{w} = \frac{\exp(s_{w})}{\exp(s_{w}) + \exp(s_{l})} = \frac{1}{1 + \exp(-(s_{w} - s_{l}))}\ ,
    \end{equation*}
    and thus $p^{w} = \sigma(s_{w} - s_{l})$. Similarly, we get that $p^{l} = \sigma(-(s_{w} - s_{l}))$. Thus, \texttt{RPO} can be seen as a generalization of \texttt{CDPO} where the soft labels are given by a certain score and not fixed. To recover the \texttt{RPO} loss we denote $a = \beta\log \frac{\pi_{\theta}(y_w|x)}{\piRef(y_w|x)} - \beta\log \frac{\pi_{\theta}(y_l|x)}{\piRef(y_l|x)}$ and $b = s_w - s_l$. Then, we get that
    \begin{equation*}
        \sigma(b) \log \frac{\sigma(b)}{\sigma(a)} + (1-\sigma(b)) \log\frac{1-\sigma(b)}{1-\sigma(a)} = p^{w}\log\frac{p^{w}}{p_{\theta}^{w}} + p^{l}\log\frac{p^{l}}{p_{\theta}^{l}}\ .
    \end{equation*}
    By eliminating the constant terms (with respect to $\theta$) $p^{w}\log p^{w} + p^{l}\log p^{l}$, we indeed get the CE loss.

    To recover the \texttt{Distilled DPO} (Equation (7) of ~\citet{distilled_DPO}), we consider the soft labels $p = \text{softmax}(s_w, s_l)$ with the loss of IPO.
    \begin{equation*}
        \mathcal{L}\big(p_{\theta},\ p\big) = \left(\log\frac{p^{w}_{\theta}}{p^{l}_{\theta}} - \log\frac{p^{w}}{p^{l}}\right)^2\ .
    \end{equation*}
    Indeed, it this case we have that
    \begin{equation*}
        \frac{p^{w}}{p^{l}} = \frac{\exp(s_{w})}{\exp(s_{w}) + \exp(s_{l})} \cdot \frac{\exp(s_{w}) + \exp(s_{l})}{\exp(s_{l})} = \frac{\exp(s_{w})}{\exp(s_{l})}\ ,
    \end{equation*}
    and therefore $\log (p^{w}/p^{l}) = s_{w} - s_{l}$. Combining this with \eqref{IPOfrac} yields the desired result.

\subsection{List of Preferences}

Our classification framework can be extended to work with lists, rather than pairs, of preferences. In particular, assume that we have $N$ responses for each prompt $x$, giving us a dataset of the form $\mathcal{D} = \{ (x,y_1, y_2,\ldots,y_N) \}.$ In this case, we can define a list version of the probability vector $p_{\theta}(x,y_1,\ldots, y_N)$ similarly to~\eqref{eq:softmax_binary}, 
%
%
together with a target distribution $p$.
With this simple extension we can now incorporate existing DPO-style algorithms that work with lists. For instance, we can show that \texttt{DPO} with the Plackett-Luce model for preferences~\cite{DPO} can be captured in our classification framework, again using hard labels and the CE loss. See Appendix \ref{Appendix:DPO_PL} for a proof.

\subsection{Auxiliary Information}

A second important extension pertains to the definition of soft labels in our classification setting. So far we have only worked with soft labels that are fixed across the entire dataset, for instance, $p \assign (p^{w}, p^{l}) = (\sigma(1/2) , \sigma(-1/2))$ for all $(x, y_w, y_l)$ in \texttt{IPO}. These fixed labels are agnostic about any extra information we may have about our data triplets $(x , y_w , y_l)$. However, in some applications we may have access to some auxiliary scores, $s_w,s_l$, (e.g.,~ratings) associated with each response, which can then be used to enrich our soft labels. 

More formally, suppose now that our dataset is comprised of ${\mathcal D} = (x,y_w, s_w, y_l, s_l)$. To obtain the soft labels we can employ, for instance, $p = \text{softmax}(s_w, s_l)$. Combining this with the \texttt{IPO} loss, we recover \texttt{Distilled DPO} (Eq.~7 in~\citealt{distilled_DPO}). Using the same soft labels, but with the CE loss recovers \texttt{RPO} (see Sec.~3.3.2 in~\citealt{nemotron}). We provide more details on both algorithms in Appendix~\ref{Appendix:RPO}. These natural extensions further demonstrate that our classification framework is fairly general as well as sufficiently flexible to capture a large number of existing DPO-style algorithms.

\section{A Concrete Example Showing that the Original \texttt{DPO}-style Problem is Under-constrained}
\label{appendix_concrete_example}
Below we give a concrete example further highlighting that the winner-loser probabilities can move in arbitrary directions.
\begin{remark}[A Concrete Example]
Suppose we have $\varepsilon=1/11$ and for simplicity we set $\beta=1$, $\piRef(y_w|x) = 0.02$, and $\piRef(y_l|x) = 0.01$, which result in $\eta = 20$. (See Eqn~\ref{DPO_solution_characterization} for the definition of $\eta$) We identify two pairs of probabilities that satisfy~\eqref{DPO_solution_characterization}: 
\begin{enumerate}
    \item $\big(\pi_{\theta^{*}}(y_w|x),\ \pi_{\theta^{*}}(y_l|x)\big) = (0.4,\ 0.02) \ .$
    \item $\big(\pi_{\theta^{*}}(y_w|x),\ \pi_{\theta^{*}}(y_l|x)\big) = (0.001,\ 0.0002)\ .$
\end{enumerate}
In the first case both probabilities (including the loser) increase under $\pi_{\theta^{*}}$ relative to $\piRef$. In sharp contrast, both probabilities (including the winner) decrease under $\pi_{\theta^{*}}$ relative to $\piRef$. Note that the increase is bounded as the two probabilities can go up until they hit $\pi_{\theta^{*}}(y_w|x) + \pi_{\theta^{*}}(y_l|x) = (1 + \eta) \pi_{\theta^{*}}(y_l | x) = 1$. Perhaps more concerning is the observation that the two probabilities can decrease arbitrarily and even collapse to $0$ while still maintaining ${\mathcal L}\big(p_{\theta^{*}},p\big) = 0$.
\end{remark}

\section{Proofs of Section~\ref{sec:algos}}
\label{Appendix:C2-DPO}

\subsection{Proof of Proposition~\ref{prop:monotonic}}
\label{Appendix:C2-DPO-Proposition}

\tech*
\begin{proof}
For simplicity we assume that $\varphi$ is monotonically increasing (the same arguments can be easily applied to the monotonically decreasing case). 
%
%
Since $\varepsilon < 1/2$, we have that $(1 - \varepsilon)/\varepsilon > 1$, and therefore from~\eqref{DPO_solution_characterization1}, we get that
%
\begin{equation}
 \pi_{\theta^{*}}(y_w|x) > \frac{\pi_{\theta^*}(y_l|x)}{\piRef(y_l|x)} \cdot \piRef(y_w|x) \ .
    \label{eq:base}
\end{equation}
We will show that $\pi_{\theta^*}(y_l|x)<\piRef(y_l|x)$ by contradiction. To this end, we assume that $\pi_{\theta^*}(y_l|x)\geq\piRef(y_l|x)$. 
From \eqref{eq:base}, we immediately obtain that $\pi_{\theta^*}(y_w|x) > \piRef(y_w|x)$. Applying the constraint function $\varphi$ to the following two inequalities:
\begin{equation*}
    \pi_{\theta^*}(y_w|x) > \piRef(y_w|x) \quad \text{and} \quad \pi_{\theta^*}(y_l|x) \geq \piRef(y_l|x) \ ,
\end{equation*}
%
and using its monotonicity, we obtain
\begin{equation*}
    \varphi\big(\pi_{\theta^{*}}(y_w|x)\big) > \varphi\big(\piRef(y_w|x)\big) \quad \text{and} \quad \varphi\big(\pi_{\theta^{*}}(y_l|x)\big) \geq \varphi\big(\piRef(y_l|x)\big).
\end{equation*}
%
By adding both sides of the above two inequalities, we get a contradiction to \eqref{P:Cons}. Thus, proving that $\pi_{\theta^*}(y_l|x)<\piRef(y_l|x)$. Similarly, we can prove that $\pi_{\theta^{*}}(y_w|x)>\piRef(y_w|x)$. 
\end{proof}

\subsection{Proof of Lemma~\ref{L:Technical}} 
\label{Appendix:C2-DPO-Lemma}

    \technical*
    \begin{proof}
        First, we write
        \begin{equation*}
            a + b = a\left(\frac{a}{a + b}\right)^{-1}\ ,
        \end{equation*}
        which tanks to classical logarithmic rules yields that
        \begin{equation*}
            \log(a + b) = \log a - \log \frac{a}{a + b} = \log a - \log \frac{1}{1 + b/a}\ .
        \end{equation*}
        Using the definition of the sigmoid function we get
        \begin{equation*}
            \frac{1}{1 + b/a} = \frac{1}{1 + \exp(\log b - \log a)}= \sigma(\log a - \log b)\ ,
        \end{equation*}
        which proves the desired result.
    \end{proof}

\section{\texttt{C2-DPO} Implementation Details} \label{Appendix:c-3dpo_pseudo_code}
    We show implementation details of \texttt{C2-DPO-Log-$\ell_1$}, \texttt{C2-DPO-Log-$\ell_2$}, \texttt{C2-DPO-$\mathcal{I}$-$\ell_1$}, and \texttt{C2-DPO-$\mathcal{I}$-$\ell_2$} below. 

    \begin{verbatim}
def dpo_loss(pi_logps, ref_logps, yw_idxs, yl_idxs, beta, algo_name, reg_coeff):
    """
    pi_logps: policy logprobs, shape (B,)
    ref_logps: reference model logprobs, shape (B,)
    yw_idxs: preferred completion indices in [0, B-1], shape (T,)
    yl_idxs: dispreferred completion indices in [0, B-1], shape (T,)
    beta: temperature controlling strength of KL penalty
    Each pair of (yw_idxs[i], yl_idxs[i]) represents the indices of a single
    preference pair.
    """
    pi_yw_logps,  pi_yl_logps =  pi_logps[yw_idxs],  pi_logps[yl_idxs]
    ref_yw_logps, ref_yl_logps = ref_logps[yw_idxs], ref_logps[yl_idxs]
    pi_logratios  = pi_yw_logps - pi_yl_logps
    ref_logratios = ref_yw_logps - ref_yl_logps
    losses = -F.logsigmoid(beta * (pi_logratios - ref_logratios))
    if algo_name == 'c2dpo_log_l1':
        reguralized_losses_without_square = \
            pi_yw_logps + pi_yl_logps - ref_yw_logps - ref_yl_logps
        reguralized_losses = (reguralized_losses_without_square) ** 2
        losses = losses + reg_coeff * reguralized_losses
    elif algo_name == 'c2dpo_log_l2':
        reguralized_losses = F.l1_loss(
            pi_yw_logps + pi_yl_logps,
            ref_yw_logps + ref_yl_logps
        )
        losses = losses + reg_coeff * reguralized_losses
    elif algo_name == 'c2dpo_i_l1':
        reguralized_losses_without_square = \
            pi_yw_logps - F.logsigmoid(pi_yw_logps - pi_yl_logps) - \
                (ref_yw_logps - F.logsigmoid(ref_yw_logps - ref_yl_logps))
        reguralized_losses = (reguralized_losses_without_square) ** 2
        losses = losses + reg_coeff * reguralized_losses
    elif algo_name == 'c2dpo_i_l2':
        reguralized_losses = F.l1_loss(
            pi_yw_logps - F.logsigmoid(pi_yw_logps - pi_yl_logps),
            ref_yw_logps - F.logsigmoid(ref_yw_logps - ref_yl_logps)
        )
        losses = losses + reg_coeff * reguralized_losses
    rewards = beta * (pi_logps - ref_logps).detach()
    return losses, rewards
    \end{verbatim}

\section{\textrm{Ultrafeedback Binarized} Claude 3.5 Sonnet v2 win rate prompt and hyperparameters} 
 \label{Appendix:win_rates_prompt}
    In this section we include the prompt used to generate win rates for the \textrm{Ultrafeedback Binarized} experiments. We use Claude Sonnet 3.5 v2 (AWS Bedrock model ID anthropic.claude-3-5-sonnet-20241022-v2:0) to generate win rates. We set the max\_tokens to 1024, temperature to 0, and used default value 0.999 for top\_p and top\_k disabled.
    \begin{verbatim}
For the following query to a chatbot, which response is more helpful?

Query: <prompt>

Response A:
<one of the responses>

Response B:
<the other response>

FIRST provide a one-sentence comparison of the two responses and explain
which you feel is more helpful. SECOND, on a new line, state only "A" or "B"
to indicate which response is more helpful. Your response should use the format:
Comparison: <one-sentence comparison and explanation>
More helpful: <"A" or "B">
    \end{verbatim}

\section{\textrm{Reddit TL;DR} Claude 3.5 Sonnet v2 win rate prompt and hyperparameters} 
 \label{Appendix:win_rates_prompt_tldr}
    In this section we include the prompt used to generate win rates for the \textrm{Reddit TL;DR} experiments. We use Claude Sonnet 3.5 v2 (AWS Bedrock model ID anthropic.claude-3-5-sonnet-20241022-v2:0) to generate win rates. We set the max\_tokens to 1024, temperature to 0, and used default value 0.999 for top\_p and top\_k disabled.
    \begin{verbatim}
Which of the following summaries does a better job of summarizing the most
important points in the given forum post, without including unimportant or
irrelevant details? A good summary is both precise and concise.
Post: {prompt}
Summary A:
{baseline_response}
Summary B:
{generated_response}
FIRST provide a one-sentence comparison of the two summaries, explaining 
which you prefer and why. SECOND, on a new line, state only "A" or "B" to 
indicate your choice. Your response should use the format:
Comparison: <one-sentence comparison and explanation>
Preferred: <"A" or "B">
    \end{verbatim}
\begin{figure*}[h]
    \centering
    \includegraphics[width=\textwidth]{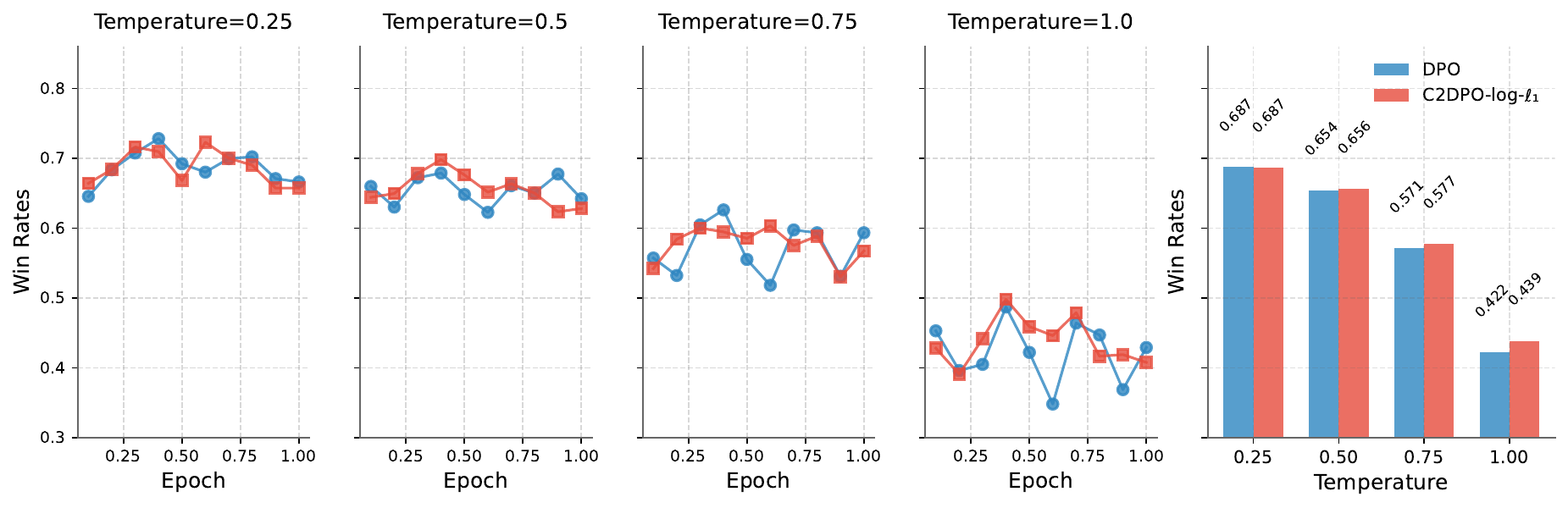}
    \caption{}
    \label{zzzz}
\end{figure*}
\section{Additional win rates analysis of \texttt{C-3DPO} with \textrm{Zephyr-7B-SFT} aligned on \textrm{Ultrafeedback Binarized}} \label{win_rates_across_10_inferences_7b_ultrafeedback}
\begin{figure*}[h]
    \centering
    \includegraphics[width=\textwidth]{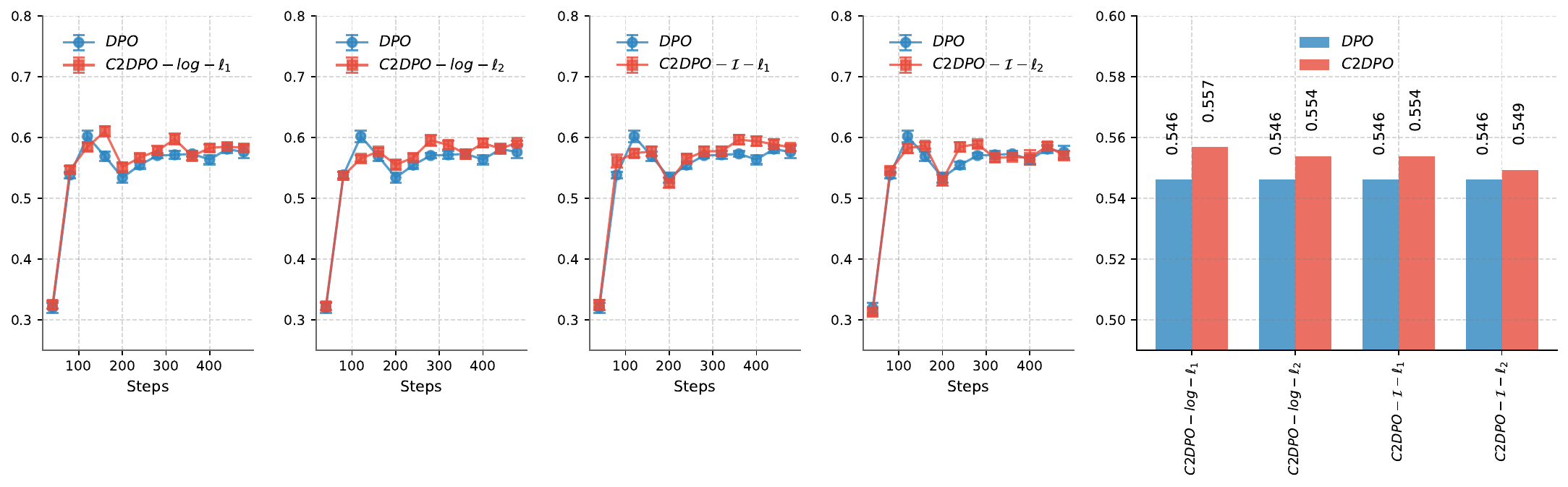}
    \caption{Win rates comparison of \textrm{Zephyr-7B-SFT} aligned on \textrm{Ultrafeedback Binarized} using \texttt{DPO} and \texttt{C-3DPO}. The first 4 plots show win rates of \texttt{C-2DPO} at individual checkpoints across 10 different inference runs. The last plot shows mean and standard error of win rates across all checkpoints and all inference runs.}
    \label{fig:ultrafeedback_zp7b_win_rates}
\end{figure*}
Figure \ref{fig:ultrafeedback_zp7b_win_rates} shows win rates comparison between \textrm{Zephyr-7B-SFT} aligned with \texttt{DPO} and \texttt{C-3DPO}. We align \textrm{Zephyr-7B-SFT} following~\citet{rasul2024preference} using \texttt{DPO}, \texttt{C-3DPO-Log-$\ell_1$}, \texttt{C-3DPO-Log-$\ell_2$}, \texttt{C-3DPO-$\mathcal{I}$-$\ell_1$}, and \texttt{C-3DPO-$\mathcal{I}$-$\ell_2$} for one epoch, all \texttt{C-3DPO} use hyper-parameter $\lambda=2\times 10^{-4}$. With each checkpoint, we generate responses using test split of \textrm{Ultrafeedback Binarized} using hyperparameters max\_tokens=1000, temperature=1.0, top\_p=0.9, top\_k=50. Different from the head to head setting, we ask Claude to compare the generated response directly with the preferred response in the dataset. The win rates and standard errors are calculated based on 10 different inference runs. 

\clearpage
\section{Limitation}
We identify two limitations in our work:

First, in analyzing various optimization problems, we look at the optimal solution and study its properties. However, this is a major simplification, because in practice we rely on variants of the gradient descent optimizer which generally speaking only finds a locally-optimal solution.

Second, while we did our best to experiment with various model sizes and datasets, our current experimental scope spans across models up to 13B parameters and 2 different datasets.
\section*{NeurIPS Paper Checklist}

\begin{enumerate}

\item {\bf Claims}
    \item[] Question: Do the main claims made in the abstract and introduction accurately reflect the paper's contributions and scope?
    \item[] Answer: \answerYes{} 

\item {\bf Limitations}
    \item[] Question: Does the paper discuss the limitations of the work performed by the authors?
    \item[] Answer: \answerYes{} 

\item {\bf Theory assumptions and proofs}
    \item[] Question: For each theoretical result, does the paper provide the full set of assumptions and a complete (and correct) proof?
    \item[] Answer: \answerYes{}

    \item {\bf Experimental result reproducibility}
    \item[] Question: Does the paper fully disclose all the information needed to reproduce the main experimental results of the paper to the extent that it affects the main claims and/or conclusions of the paper (regardless of whether the code and data are provided or not)?
    \item[] Answer: \answerYes{}

\item {\bf Open access to data and code}
    \item[] Question: Does the paper provide open access to the data and code, with sufficient instructions to faithfully reproduce the main experimental results, as described in supplemental material?
    \item[] Answer: \answerNo{}
    \item[] Justification: We are unable to release code at this time due to institutional or organizational IP constraints. We can not elaborate on this due to anonymity concerns. We hope to explore possibilities for a future release, pending approval.

\item {\bf Experimental setting/details}
    \item[] Question: Does the paper specify all the training and test details (e.g., data splits, hyperparameters, how they were chosen, type of optimizer, etc.) necessary to understand the results?
    \item[] Answer:  \answerYes{}

\item {\bf Experiment statistical significance}
    \item[] Question: Does the paper report error bars suitably and correctly defined or other appropriate information about the statistical significance of the experiments?
    \item[] Answer: \answerYes{} 
    \item[] Justification: Note that we only do so for the main experiment of the paper (Figure 3). Reporting error bars for all experiments required learning tens of runs of LLM as the judge (with extremely large closed-weight LLMs) which was expensive to perform.

\item {\bf Experiments compute resources}
    \item[] Question: For each experiment, does the paper provide sufficient information on the computer resources (type of compute workers, memory, time of execution) needed to reproduce the experiments?
    \item[] Answer: \answerYes{} 
    \item[] We conducted all of our experiments on 8 Nvidia H100 GPUs.
    
\item {\bf Code of ethics}
    \item[] Question: Does the research conducted in the paper conform, in every respect, with the NeurIPS Code of Ethics 
    \item[] Answer: \answerYes{}

\item {\bf Broader impacts}
    \item[] Question: Does the paper discuss both potential positive societal impacts and negative societal impacts of the work performed?
    \item[] Answer: \answerNA{}
    \item[] Justification: The focus in this paper is to improve the core algorithmic ideas in preference optimization. The aim is to make AI more aligned with human intent, but like any other powerful technology, these algorithms should always be used with adequate care and due diligence.
    
\item {\bf Safeguards}
    \item[] Question: Does the paper describe safeguards that have been put in place for responsible release of data or models that have a high risk for misuse (e.g., pretrained language models, image generators, or scraped datasets)?
    \item[] Answer: \answerNA{}

\item {\bf Licenses for existing assets}
    \item[] Question: Are the creators or original owners of assets (e.g., code, data, models), used in the paper, properly credited and are the license and terms of use explicitly mentioned and properly respected?
    \item[] Answer: \answerNA{}

\item {\bf New assets}
    \item[] Question: Are new assets introduced in the paper well documented and is the documentation provided alongside the assets?
    \item[] Answer: \answerNA{}

\item {\bf Crowdsourcing and research with human subjects}
    \item[] Question: For crowdsourcing experiments and research with human subjects, does the paper include the full text of instructions given to participants and screenshots, if applicable, as well as details about compensation (if any)? 
    \item[] Answer: \answerNA{}

\item {\bf Institutional review board (IRB) approvals or equivalent for research with human subjects}
    \item[] Question: Does the paper describe potential risks incurred by study participants, whether such risks were disclosed to the subjects, and whether Institutional Review Board (IRB) approvals (or an equivalent approval/review based on the requirements of your country or institution) were obtained?
    \item[] Answer: \answerNA{}

\item {\bf Declaration of LLM usage}
    \item[] Question: Does the paper describe the usage of LLMs if it is an important, original, or non-standard component of the core methods in this research? Note that if the LLM is used only for writing, editing, or formatting purposes and does not impact the core methodology, scientific rigorousness, or originality of the research, declaration is not required.
    \item[] Answer: \answerYes{} 
    \item[] Justification: As mentioned in the experimental section, we used LLM as the judge to asses the performance of our proposed algorithm against the baseline. This has become the standard practice in the prference optimization literature.

\end{enumerate}
\end{document}